\documentclass[nohyperref]{article}

\usepackage[accepted]{icml2022}

\usepackage{amsmath}
\usepackage{amssymb}
\usepackage{mathtools}
\usepackage{amsthm}
\usepackage{multirow}
\usepackage{hyperref}
\usepackage{url}
\usepackage{booktabs}
\usepackage{amsfonts}       % blackboard math symbols
\usepackage{nicefrac}       % compact symbols for 1/2, etc.
\usepackage{microtype}      % microtypography
\usepackage{xcolor}         % colors
\usepackage{graphicx}
\usepackage{caption}
\usepackage{colortbl}
\usepackage{color}
\usepackage{xcolor}
\usepackage{wrapfig}
\usepackage{lscape} 
\usepackage{makecell}

\usepackage[capitalize,noabbrev]{cleveref}

\theoremstyle{plain}
\newtheorem{prop}{Proposition}
\newtheorem{cor}{Corollary}
\newtheorem{thm}{Theorem}

\newcommand{\STAB}[1]{\begin{tabular}{@{}c@{}}#1\end{tabular}}

\usepackage[textsize=tiny]{todonotes}

\icmltitlerunning{Learning Stable Classifiers by Transferring Unstable Features}

\begin{document}

\twocolumn[
\icmltitle{Learning Stable Classifiers by Transferring Unstable Features}

% It is OKAY to include author information, even for blind
% submissions: the style file will automatically remove it for you
% unless you've provided the [accepted] option to the icml2022
% package.

% List of affiliations: The first argument should be a (short)
% identifier you will use later to specify author affiliations
% Academic affiliations should list Department, University, City, Region, Country
% Industry affiliations should list Company, City, Region, Country

% You can specify symbols, otherwise they are numbered in order.
% Ideally, you should not use this facility. Affiliations will be numbered
% in order of appearance and this is the preferred way.
\icmlsetsymbol{equal}{*}
  
\begin{icmlauthorlist}
\icmlauthor{Yujia Bao}{csail}
\icmlauthor{Shiyu Chang}{ucsb}
\icmlauthor{Regina Barzilay}{csail}
\end{icmlauthorlist}

\icmlaffiliation{csail}{MIT CSAIL}
\icmlaffiliation{ucsb}{Computer Science, UC Santa Barbara}

\icmlcorrespondingauthor{Yujia Bao}{yujia@csail.mit.edu}

% You may provide any keywords that you
% find helpful for describing your paper; these are used to populate
% the "keywords" metadata in the PDF but will not be shown in the document
\icmlkeywords{Machine Learning, ICML}

\vskip 0.3in
]

% this must go after the closing bracket ] following \twocolumn[ ...

% This command actually creates the footnote in the first column
% listing the affiliations and the copyright notice.
% The command takes one argument, which is text to display at the start of the footnote.
% The \icmlEqualContribution command is standard text for equal contribution.
% Remove it (just {}) if you do not need this facility.

\printAffiliationsAndNotice{}  % leave blank if no need to mention equal contribution
%\printAffiliationsAndNotice{\icmlEqualContribution} % otherwise use the standard text.

\begin{abstract}
While unbiased machine learning models are essential for many applications,
bias is a human-defined concept that can vary across tasks. Given only input-label pairs, algorithms may lack sufficient information to distinguish stable (causal) features from unstable (spurious) features. However, related tasks often share similar biases -- an observation we may leverage to develop stable classifiers in the transfer setting. In this work, we explicitly inform the target classifier about unstable features in the source tasks. Specifically, we derive a representation that encodes the unstable features by contrasting different data environments in the source task. We achieve robustness by clustering data of the target task according to this representation and minimizing the worst-case risk across these clusters. We evaluate our method on both text and image classifications. Empirical results demonstrate that our algorithm is able to maintain robustness on the target task for both synthetically generated environments and real-world environments. Our code is available at \url{https://github.com/YujiaBao/Tofu}.
\end{abstract}

\section{Introduction}
Automatic de-biasing~\citep{sohoni2020no,creager2021environment,sanh2021learning} has emerged as a promising direction for learning stable classifiers. The key premise here is that no additional annotations for the bias attribute are required. However, bias is a human-defined concept and can vary from task to task. Provided with only input-label pairs, algorithms may not have sufficient information to distinguish stable (causal) features from unstable (spurious) features.

To address this challenge, we note that related tasks are often fraught with similar spurious correlations. For instance, when classifying animals such as camels vs.\ cows, their backgrounds (desert vs.\ grass) may constitute a spurious correlation~\citep{beery2018recognition}. The same bias between the label and the background also persists in other related classification tasks (such as sheep vs.\ antelope). In the resource-scarce target task, we only have access to the input-label pairs. However, in the source tasks, where training data is sufficient, identifying biases may be easier. For instance, we may have examples collected from multiple environments, in which correlations between bias features and the label are different~\citep{arjovsky2019invariant}. These source environments help us define the exact bias features that we want to regulate.

\begin{figure*}[t]
\centering
\includegraphics[width=0.9\linewidth]{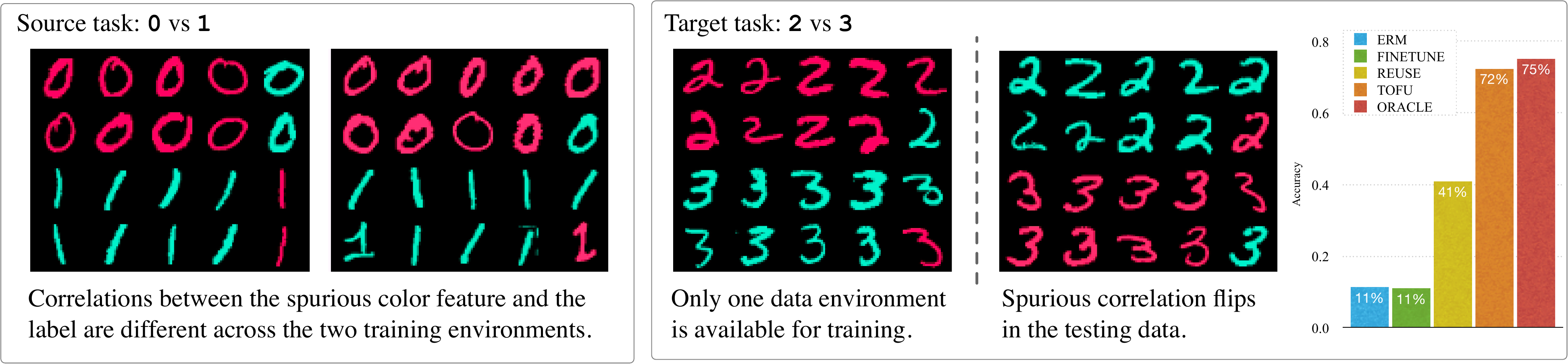}
\caption{Transferring across tasks in Colored MNIST~\citep{arjovsky2019invariant}. On the source task, we learn a color-invariant model that achieves oracle performance (given direct access to the unstable features). However, directly transferring this model to the target task, by reusing or fine-tuning its feature extractor, severely overfits the spurious correlation and underperforms the majority baseline (50\%) on a test set where the spurious correlation flips. By explicitly transferring the unstable features, our algorithm \textsc{tofu} (Transfer OF Unstable features) is able to reach the oracle performance.}
\label{fig:intro}
\end{figure*}

One obvious approach to utilize the source task is direct transfer. Specifically, given multiple source environments, we can train an unbiased source classifier and then apply its representation to the target task. However, we empirically demonstrate that while the source classifier is not biased when making its final predictions, its internal continuous representation can still encode information about the unstable features. Figure~\ref{fig:intro} shows that in Colored MNIST, where the digit label is spuriously correlated with the image color, direct transfer by either re-using or fine-tuning the representation learned on the source task fails in the target task, performing no better than the majority baseline.

In this paper, we propose to explicitly inform the target classifier about unstable features from the source data. Specifically, we derive a representation that encodes these unstable features using the source environments. Then we identify distinct subpopulations by clustering examples based on this representation and apply group DRO~\citep{sagawa2019distributionally} to minimize the worst-case risk over these subpopulations. As a result, we enforce the target classifier to be robust against different values of the unstable features. In the example above, animals would be clustered according to backgrounds, and the classifier should perform well regardless of the clusters (backgrounds).

The remaining question is how to compute the unstable feature representation using the source data environments. Following \citet{yujia2021predict},  we hypothesize that unstable features are reflected in mistakes observed during classifier transfer across environments. For instance, if the classifier uses the background to distinguish camels from cows, the camel images that are predicted correctly would have a desert background while those predicted incorrectly are likely to have a grass background. More generally, we prove that among examples with the same label value, those with the same prediction outcome will have more similar unstable features than those with different predictions. By forcing examples with the same prediction outcome to stay closer in the feature space, we obtain a representation that encodes these latent unstable features.

We evaluate our approach, Transfer OF Unstable features (\textsc{tofu}), on both synthetic and real-world environments. Our synthetic experiments first confirm our hypothesis that standard transfer approaches fail to learn a stable classifier for the target task. By explicitly transferring the unstable features, our method significantly improves over the best baseline across 12 transfer settings (22.9\% in accuracy), and reaches the performance of an oracle model that has direct access to the unstable features (0.3\% gap). Next, we consider a practical setting where environments are defined by an input attribute and our goal is to reduce biases from other unknown attributes. On CelebA, \textsc{tofu} achieves the best worst-group accuracy across 38 latent attributes, outperforming the best baseline by 18.06\%. Qualitative and quantitative analyses confirm that \textsc{tofu} is able to identify the unstable features.
\section{Related work}\label{sec:related}

\paragraph{Removing bias via annotations:}
Due to idiosyncrasies of the data collection process, annotations are often coupled with unwanted biases~\citep{buolamwini2018gender,schuster2019towards,mccoy-etal-2019-right,yang2019analyzing}.
To address this issue and learn robust models, researchers leverage extra information~\citep{belinkov2019don,stacey-etal-2020-avoiding,hinton2002training,clark2019don,he2019unlearn, mahabadi2020end}.
One line of work assumes that the bias attributes are known and have been annotated for each example, e.g., group distributionally robust optimization (DRO)~\citep{hu2018does, oren2019distributionally, Sagawa*2020Distributionally}. By defining groups based on these bias attributes, we explicitly specify the distribution family to optimize over. However, identifying the hidden biases is time-consuming and often requires domain knowledge~\citep{zellers-etal-2019-hellaswag,sakaguchi2020winogrande}. To address this issue, another line of work~\citep{peters2016causal,krueger2020outofdistribution,chang2020invariant,jin2020enforcing,ahuja2020invariant,arjovsky2019invariant,yujia2021predict,kuang2020stable,shen2020stable} only assumes access to a set of data environments. These environments are defined based on readily-available information of the data collection circumstances, such as location and time. The main assumption is that while spurious correlations vary across different environments, the association between the causal features and the label should stay the same. Thus, by learning a representation that is invariant across all environments, they alleviate the dependency on spurious features. 
In contrast to previous works, we don't have access to any additional information besides the labels in our target task. We show that we can achieve robustness by transferring the unstable features from a related source task.
%Different from the previous work, in this paper we leverage knowledge from other tasks. We show that even though we don't have access to any additional information in the target task, we can still achieve robustness by transfer.

\begin{figure*}[t]
  \centering
  \includegraphics[width=0.9\linewidth]{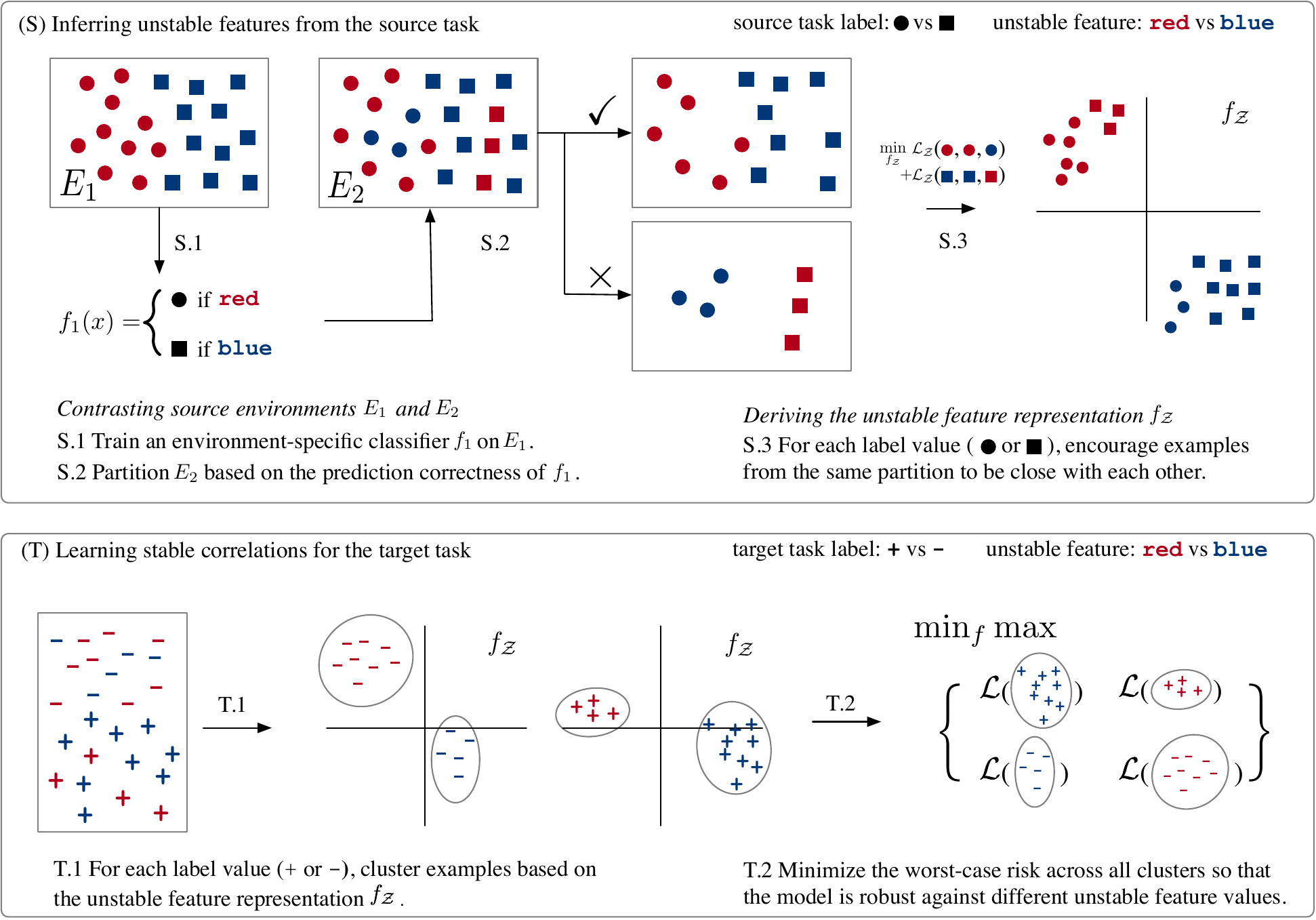}
  \caption{Our algorithm \textsc{tofu} 1) infers unstable features from the source task (Section~\ref{sec:source}) and 2) learns stable correlations for the target task (Section~\ref{sec:target}). We create partitions for all environment pairs.
  For ease of illustration, we only depict using $f_1$ to partition $E_2$. Best viewed in color.  
  }
  \label{fig:model}
\end{figure*}

\paragraph{Automatic de-biasing}
A number of recent approaches focus on a more common setting where the algorithm only has access to the input-label pairs. \citet{li-2021-discover} leverages  disentangled representations in generative models to identify biases.
\citet{sanh2021learning, nam2020learning, utama2020towards} find that weak models are more vulnerable to spurious correlations as they only learn shallow heuristics. By boosting from their mistakes, they obtain a more robust model. \citet{qiao2020learning} uses adversarial learning to augment the biased training data. \citet{creager2021environment,sohoni2020no,ahmed2020systematic,matsuura2020domain,liu2021heterogeneous} propose to identify minority groups by looking at the features produced by a biased model.

These automatic approaches are intriguing as they do not require additional annotation. However, we note that bias is a human-centric concept and can vary from tasks to tasks. For models that only have access to the input-label pairs, they have no information to distinguish causal features from bias features.
For example, consider the Colored MNIST dataset, where color and digit shape are correlated in the training set but not in the test set. 
If our task is to predict the digit, then color becomes the spurious bias that we want to remove. Vice versa, if we want to predict the color, then digit shape is spurious.
\citet{creager2021environment,nam2020learning} empirically demonstrate that they can learn a color-invariant model for the digit prediction task. However, their approaches will result in the same color-invariant model for the color prediction task, and thus fail at test time, when color and digit are no longer correlated.
In this work, we leverage source tasks to define the exact bias that we want to remove for the target task.

\paragraph{Transferring robustness across tasks:}
Prior work has also studied the transferability of adversarial robustness across tasks. For example, \citet{hendrycks2019using, Shafahi2020Adversarially} show that by pre-training the model on a large-scale source task, we can improve the model robustness against adversarial perturbations over $l_\infty$ norm. We note that these perturbations measure the smoothness of the classifier, rather than the stability of the classifier against spurious correlations. In fact, our results show that if we directly re-use or fine-tune the pre-trained feature extractor on the target task, the model will quickly over-fit to the unstable correlations present in the data. We propose to address this issue by explicitly inferring the unstable features using the source environments and use this information to guide the target classifier during training.

\section{Method}\label{sec:method}
\paragraph{Problem formulation}
We consider the transfer problem from a source task to a target task.
For the source task, we assume the standard setting~\citep{arjovsky2019invariant} where the training data contain $n$ environments $E_1, \ldots, E_n$.
Within each environment $E_i$, examples are drawn from the joint distribution $P_i (x, y)$.
Following \citet{woodward2005making}, we define unstable features $\mathcal{Z}(x)$ as features that are \emph{differentially} correlated with the label across the environments.
We note that $\mathcal{Z}(x)$ is unknown to the model.

For the target task, we only have access to the input-label pairs $(x, y)$ (i.e. no environments).
We assume that the target label is \emph{not} causally associated with the above unstable features $\mathcal{Z}$.
However, due to collection biases, the target data may contain \emph{spurious correlations} between the label and $\mathcal{Z}$.
Our goal is to transfer the knowledge that $\mathcal{Z}$ is unstable in the source task, so that the target classifier will not rely on these spurious features.

\paragraph{Overview}
If the unstable features have been identified for the target task, we can simply apply group DRO to learn a stable classifier. By grouping examples based on the unstable features and minimizing the worst-case risk over these \emph{manually-defined} groups, we explicitly address the bias from these unstable features~\citep{hu2018does, oren2019distributionally, Sagawa*2020Distributionally}. In our setup, while these unstable features are not accessible, we can leverage the source environments to derive groups over the target data that are informative of these biases.
Applying group DRO on these \emph{automatically-derived} groups, we can eliminate the unstable correlations in the target task.

%to \emph{automatically derive} groups that are informative of the unstable features for the target data. \shiyu{leverage the source environments to derive groups on target task that are able to eliminate unstable features when applying group DRO?}

Our overall transfer paradigm is depicted in Figure~\ref{fig:model}. It consists of two steps: inferring unstable features from the source task (Section \ref{sec:source}) and learning stable correlations for the target task (Section \ref{sec:target}). First, for the source task we use a classifier trained on one environment to partition data from another environment based on the correctness of its predictions. Starting from the theoretical results in \citep{yujia2021predict}, we show that these partitions reflect the similarity of the examples in terms of their unstable features: among examples with the same label value, those that share the same prediction outcome have more similar unstable features than those with different predictions (Theorem~\ref{thm:1}). We can then derive a representation $f_\mathcal{Z}$ where examples are distributed based on the unstable features $\mathcal{Z}$. Next, we cluster target examples into groups based on the learned unstable feature representation $f_\mathcal{Z}$. These \emph{automatically-derived} groups correspond to different modes of the unstable features, and they act as proxies to the \emph{manually-defined} groups in the oracle setting where unstable features are explicitly annotated. Finally, we use group DRO to obtain our robust target classifier by minimizing the worst-case risk over these groups.

\subsection{Inferring unstable features from the source task}
\label{sec:source}
Given the data environments from the source task, we would like to 1) identify the unstable correlations across these environments; 2) learn a representation $f_{\mathcal{Z}}(x)$ that encodes the unstable features $\mathcal{Z}(x)$. We achieve the first goal by contrasting the empirical distribution of different environments (Figure~\ref{fig:model}.S.1 and Figure~\ref{fig:model}.S.2) and the second goal by metric learning (Figure~\ref{fig:model}.S.3).

Let $E_i$ and $E_j$ be two different data environments. \citet{yujia2021predict} shows that by training a classifier $f_i$ on $E_i$ and using it to make predictions on $E_j$, we can reveal the unstable correlations from its prediction results. Intuitively, if the unstable correlations are stronger in $E_i$, the classifier $f_i$ will overuse these correlations and make mistakes on $E_j$ when these stronger correlations do not hold.

In this work, we connect the prediction results directly to the unstable features. We show that the prediction results of the classifier $f_i$ on $E_j$  estimate the relative distance of the unstable features.
\begin{thm}[Simplified]\label{thm:1}
Consider examples in $E_j$ with label value $y$. Let $X_1^\checkmark, X_2^\checkmark$ denote two batches of examples that $f_i$ predicted correctly, and let $X_3^\times$ denote a batch of incorrect predictions. We use $\overline{\phantom{\,}\cdot\phantom{\,}}$ to represent the mean across a given batch. Following the same assumption in \citep{yujia2021predict}, we have 
\[
\| \overline{\mathcal{Z}}(X_1^\checkmark) - \overline{\mathcal{Z}}(X_2^\checkmark) \|_2
< \| \overline{\mathcal{Z}}(X_1^\checkmark) - \overline{\mathcal{Z}}(X_3^\times) \|_2
\]
almost surely for large enough batch size.\footnote{See Appendix~\ref{app:theory} for the full theorem and proof.}
\end{thm}
The result makes intuitive sense as we would expect example pairs that share the same prediction outcome should be more similar than those with different prediction outcomes.
We note that it is critical to look at examples with the same label value; otherwise, the unstable features will be coupled with the task-specific label in the prediction results.

While the value of the unstable features $\mathcal{Z}(x)$ is still not directly accessible, Theorem~\ref{thm:1} enables us to learn a feature representation $f_{\mathcal{Z}}(x)$ that preserves the distance between the examples in terms of their unstable features. 
We adopt standard metric learning~\citep{chechik2010large} to minimize the following triplet loss:
\begin{equation}
\begin{aligned}
\label{eq:hinge}
\mathcal{L}_\mathcal{Z}(X_1^\checkmark, X_2^\checkmark, X_3^\times)  = \max(0,\, &\delta+\|\overline{f_{\mathcal{Z}}}(X_1^\checkmark) - \overline{f_{\mathcal{Z}}}(X_2^\checkmark) \|_2^2\\&
-\|\overline{f_{\mathcal{Z}}}(X_1^\checkmark) - \overline{f_{\mathcal{Z}}}(X_3^\times) \|_2^2),
\end{aligned}
\end{equation}
where $\delta$ is a hyper-parameter.
By minimizing Eq~\eqref{eq:hinge},
we encourage examples that have similar unstable features to be close in the representation $f_\mathcal{Z}$.
To summarize, inferring unstable features from the source task consists of three steps (Figure~\ref{fig:model}.S):
\begin{itemize}
\item[\textbf{S.1}] For each source environment $E_i$, train an environment-specific classifier $f_i$.
\item[\textbf{S.2}] For each pair of environments $E_i$ and $E_j$, use classifier $f_i$ to partition $E_j$ into two sets: $E_{j}^{i \checkmark}$ and $E_j^{i \times}$, where $E_{j}^{i \checkmark}$ contains examples that $f_i$ predicted correctly and $E_j^{i \times}$ contains those predicted incorrectly.
\item[\textbf{S.3}] Learn an unstable feature representation $f_{\mathcal{Z}}$ by minimizing Eq~\eqref{eq:hinge} across all pairs of environments $E_i, E_j$ and all possible label value $y$: 
  \[f_\mathcal{Z} = \arg\min
  \sum_{y, E_i\neq E_j} 
  \mathbb{E}_{X_1^\checkmark, X_2^\checkmark, X_3^\times}
  \left[\mathcal{L}_\mathcal{Z}(X_1^\checkmark, X_2^\checkmark, X_3^\times)\right],
  \]
  where batches $X_1^\checkmark, X_2^\checkmark$ are sampled uniformly from $E_j^{i\checkmark}|_y$ and batch $X_3^\times$ is sampled uniformly from $E_j^{i\times}|_y$ ($\cdot|_y$ denotes the subset of $\cdot$ with label value $y$).
\end{itemize}

\subsection{Learning stable correlations for the target task}
\label{sec:target}
Given the unstable feature representation $f_\mathcal{Z}$, our goal is to learn a target classifier that focuses on the stable correlations rather than using unstable features.
Inspired by group DRO~\citep{Sagawa*2020Distributionally}
we minimize the worst-case risk across groups of examples that are representative of different unstable feature values.
However, in contrast to DRO, these groups are constructed automatically based on the previously learned representation $f_\mathcal{Z}$.

For each target label value $y$, we use the representation $f_\mathcal{Z}$ to cluster target examples with label $y$ into different clusters (Figure~\ref{fig:model}.T.1). Since these clusters capture different modes of the unstable features, they are approximations of the typical manually-defined groups when annotations of the unstable features are available. By minimizing the worst-case risk across all clusters, we explicitly enforce the classifier to be robust against unstable correlations (Figure~\ref{fig:model}.T.2). We note that it is important to cluster within examples of the same label, as opposed to clustering the whole dataset.
Otherwise, the cluster assignment may be correlated with the target label.

Concretely, learning stable correlations for the target task has two steps (Figure~\ref{fig:model}.T).
\begin{itemize}
  \item[\textbf{T.1}] For each label value $y$, apply K-means ($l_2$ distance) to cluster examples with label $y$ in the feature space $f_\mathcal{Z}$.
  We use $C_1^y, \ldots, C_{n_c}^y$ to denote the resulting cluster assignment, where $n_c$ is a hyper-parameter.
  \item[\textbf{T.2}] Train the target classifier $f$ by minimizing the worst-case risk over all clusters:
  \[
  f = \arg \min \max_{i, y} \mathcal{L}(C_i^y),
  \]
  where $\mathcal{L}(C_i^y)$ is the empirical risk on cluster $C_i^y$.
\end{itemize}

\section{Experimental setup}\label{sec:setup}

\begin{table}[t]
  \centering
  \caption{Pearson correlation coefficient between the spurious feature $\mathcal{Z}$ and the label $Y$ for each task. The validation environment $E^{\text{val}}$ follows the same distribution as $E_1^{\text{train}}$. We study the transfer problem between different task pairs.
  For the source task $S$, the model can access $E_1^{\text{train}}(S), E_2^{\text{train}}(S)$ and $E^{\text{val}}(S)$.
  For the target task $T$, the model can access $E_1^{\text{train}}(T)$ and $E^{\text{val}}(T)$.
  \label{tab:data}}
  \small
  \begin{tabular}{cccccc}
  \toprule
  $\rho(\mathcal{Z}, Y)$ & Task & $E_1^{\text{train}}$ & $E_2^{\text{train}}$ & $E^{\text{val}}$ & $E^{\text{test}}$ \\\midrule
  \multirow{2}{*}{\textsc{mnist}}
  & \textsc{odd} & \cellcolor{red!55} 0.87 & \cellcolor{red!46} 0.75 & \cellcolor{red!55} 0.87 & \cellcolor{blue!15} -0.11\\\cmidrule{2-6}
  & \textsc{even} & \cellcolor{red!55} 0.87 & \cellcolor{red!46} 0.75 & \cellcolor{red!55} 0.87 & \cellcolor{blue!15} -0.11\\\midrule
  \multirow{3}{*}{\textsc{beer review}}
  & \textsc{look} & \cellcolor{red!40} 0.60 & \cellcolor{red!50} 0.80 & \cellcolor{red!40} 0.60 & \cellcolor{blue!47} -0.80\\\cmidrule{2-6}
  & \textsc{aroma} & \cellcolor{red!40} 0.60 & \cellcolor{red!50} 0.80 & \cellcolor{red!40} 0.60 & \cellcolor{blue!47} -0.80\\\cmidrule{2-6}
  & \textsc{palate} & \cellcolor{red!40} 0.60 & \cellcolor{red!50} 0.80 & \cellcolor{red!40} 0.60 & \cellcolor{blue!47} -0.80\\\midrule 
  \multirow{2}{*}{\textsc{ask2me}}
  & \textsc{pene.} & \cellcolor{red!20} 0.31 & \cellcolor{red!35} 0.52 & \cellcolor{red!20} 0.31 & 0.00\\\cmidrule{2-6}
  & \textsc{inci.} & \cellcolor{red!30} 0.44 & \cellcolor{red!43} 0.66 & \cellcolor{red!30}0.44 & 0.00\\\midrule
  \multirow{2}{*}{\textsc{waterbird}}
  & \textsc{water} & \cellcolor{red!24} 0.36 & \cellcolor{red!42} 0.63 & \cellcolor{red!24} 0.36 & 0.00\\\cmidrule{2-6}
  & \textsc{sea} & \cellcolor{red!26} 0.39 & \cellcolor{red!42} 0.64 & \cellcolor{red!26} 0.39 & 0.00  
  \\\bottomrule
  \end{tabular}
\end{table}
\begin{table*}[t]
  \centering
  \caption{Target task accuracy of different methods.
  %on image classification (\textsc{mnist} and \textsc{waterbird}) and text classification (\textsc{beer review} and \textsc{ask2me}).
  All methods are tuned based on a held-out validation set that follows from the same distribution as the target training data. Bottom right: standard deviation across 5 runs. Upper right: source task testing performance (if applicable).
  }\label{tab:big}
  \small
  \begin{tabular}{cccccccccc}
  \toprule
  & \textsc{source} &\textsc{target} & \textsc{erm} & $\textsc{reuse}_{\textsc{pi}}$ & $\textsc{finetune}_{\textsc{pi}}$ & \textsc{multitask} & \textsc{tofu} & \textsc{oracle} \\
  \midrule
\multirow{2}{*}{\STAB{\rotatebox[origin=c]{90}{\textsc{mnist}}}}\hspace{-1mm}
& \textsc{odd} & \textsc{even} & 12.3\tiny{$\pm0.6$} & 14.4$_{\tiny\pm1.0}^{\tiny(70.9)}$ &  11.2$_{\tiny\pm2.1}^{\tiny(70.1)}$ & 11.6{$\tiny^{(69.6)}_{\pm0.6}$} & \textbf{69.1}\tiny{$\pm1.6$} &   68.7\tiny{$\pm0.9$} \\ \cmidrule{2-9}
& \textsc{even} & \textsc{odd} & \phantom{0}9.7\tiny{$\pm0.6$} & 19.2$_{\tiny\pm2.3}^{\tiny (71.1)}$ &  11.5$_{\tiny\pm1.2}^{\tiny (71.1)}$ & 10.1$_{\tiny{\pm0.7}}^{\tiny(70.0)}$ & \textbf{66.8}\tiny{$\pm0.8$} & 67.8\tiny{$\pm0.5$} \\ \midrule
\multirow{6}{*}{\STAB{\rotatebox[origin=c]{90}{\hspace{-9mm}\textsc{beer review}}}}\hspace{-1mm}
& \textsc{look}   & \textsc{aroma}  & 55.5\tiny{$\pm1.7$} & 31.9$_{\tiny\pm1.0}^{\tiny(70.1)}$ & 53.7$_{\tiny\pm1.4}^{\tiny(70.1)}$ & 54.1$_{\tiny\pm2.2}^{\tiny(76.0)}$ & \textbf{75.9}\tiny{$\pm1.4$} & 77.3\tiny{$\pm1.3$} \\ \cmidrule{2-9}
& \textsc{look}    & \textsc{palate} & 46.9\tiny{$\pm0.3$} & 22.8$_{\tiny\pm1.9}^{\tiny(70.0)}$ & 49.3$_{\tiny\pm2.1}^{\tiny(73.2)}$ & 52.8$_{\tiny\pm2.9}^{\tiny(73.3)}$ & \textbf{73.8}\tiny{$\pm0.7$} & 74.0\tiny{$\pm1.2$} \\ \cmidrule{2-9} 
& \textsc{aroma}   & \textsc{look}   & 63.9\tiny{$\pm0.6$} & 40.1$_{\tiny\pm3.1}^{\tiny(68.6)}$ & 65.2$_{\tiny\pm1.8}^{\tiny(66.4)}$ & 64.0$_{\tiny\pm0.6}^{\tiny(71.5)}$ & \textbf{80.9}\tiny{$\pm0.5$} & 80.1\tiny{$\pm0.6$} \\ \cmidrule{2-9} 
& \textsc{aroma}   & \textsc{palate} & 46.9\tiny{$\pm0.3$} & 14.0$_{\tiny\pm2.4}^{\tiny(68.3)}$ & 47.9$_{\tiny\pm3.3}^{\tiny(63.2)}$ & 50.0$_{\tiny\pm1.4}^{\tiny(71.2)}$ & \textbf{73.5}\tiny{$\pm1.1$} & 74.0\tiny{$\pm1.2$} \\ \cmidrule{2-9} 
& \textsc{palate}  & \textsc{look}   & 63.9\tiny{$\pm0.6$} & 40.4$_{\tiny\pm2.8}^{\tiny (57.2)}$ & 64.3$_{\tiny\pm2.7}^{\tiny (60.1)}$ & 63.1$_{\tiny\pm1.0}^{\tiny(75.9)}$ & \textbf{81.0}\tiny{$\pm1.0$} & 80.1\tiny{$\pm0.6$}  \\ \cmidrule{2-9} 
& \textsc{palate}  & \textsc{aroma}  & 55.5\tiny{$\pm1.7$} & 23.1$_{\tiny \pm3.3}^{\tiny (59.2)}$ & 54.5$_{\tiny\pm1.2}^{\tiny (58.7)}$ & 56.5$_{\tiny\pm1.3}^{\tiny(73.3)}$ & \textbf{76.9}\tiny{$\pm1.5$} & 77.3\tiny{$\pm1.3$} \\\midrule
\multirow{2}{*}{\STAB{\rotatebox[origin=c]{90}{\textsc{ask.}}}}\hspace{-1mm}
& \textsc{pene}    & \textsc{inci.}  & 79.3\tiny{$\pm1.3$} & 71.7$_{\tiny \pm0.5}^{\tiny (72.7)}$ & 79.3$_{\tiny\pm0.8}^{\tiny (71.2)}$ & 71.1$_{\tiny\pm1.4}^{\tiny(73.5)}$ & \textbf{83.2}\tiny{$\pm1.8$} & 84.8\tiny{$\pm1.2$} \\\cmidrule{2-9}
& \textsc{inci.}   & \textsc{pene.}  & 71.6{\tiny$\pm1.8$} & 64.1$_{\tiny\pm1.5}^{\tiny(83.4)}$ & 72.0$_{\tiny\pm3.1}^{\tiny(83.4)}$ & 61.9$_{\tiny\pm0.7}^{\tiny(82.4)}$ & \textbf{78.1}\tiny{$\pm1.4$} & 78.3\tiny{$\pm0.9$}\\  \midrule   
\multirow{2}{*}{\STAB{\rotatebox[origin=c]{90}{\textsc{bird}}}}\hspace{-1mm}
& \textsc{water}   & \textsc{sea}    & 81.8\tiny{$\pm4.3$} & 87.8$_{\tiny\pm1.1}^{\tiny(99.5)}$ & 82.0$_{\tiny\pm4.0}^{\tiny(99.5)}$ & 88.0$_{\tiny\pm0.9}^{\tiny(99.5)}$ & \textbf{93.1}{\tiny$\pm0.4$} & 93.7{\tiny$\pm0.7$} \\\cmidrule{2-9}
& \textsc{sea}     & \textsc{water}  & 75.1\tiny{$\pm6.3$} & 94.6$_{\tiny\pm1.6}^{\tiny(93.3)}$ & 78.2$_{\tiny\pm8.1}^{\tiny(93.1)}$ & 93.5$_{\tiny\pm1.9}^{\tiny(92.7)}$ & \textbf{99.0}{\tiny$\pm0.4$} & 98.9\tiny{$\pm0.5$} \\\midrule\midrule
& \multicolumn{2}{l}{Average} &  55.2\phantom{{\tiny$\pm0.4$}} &	43.7\phantom{$_{\tiny\pm0.4}^{\tiny (10.0)}$}  &	55.8\phantom{$_{\tiny\pm0.4}^{\tiny (10.0)}$}  &	56.4\phantom{$_{\tiny\pm0.4}^{\tiny (10.0)}$}	& \textbf{79.3}\phantom{{\tiny$\pm0.4$}}	& 79.6\phantom{{\tiny$\pm0.4$}}
\\\bottomrule
  \end{tabular}
\end{table*}

\subsection{Datasets and settings}

\paragraph{Synthetic environments}
We start with controlled experiments where environments are created based on the spurious correlation. We consider four datasets: MNIST~\citep{lecun1998gradient}, BeerReview~\citep{mcauley2012learning},
ASK2ME~\citep{doi:10.1200/CCI.19.00042} and
Waterbird~\citep{sagawa2019distributionally}.
In MNIST and BeerReview, we inject spurious feature to the input (background color for MNIST and pseudo token for BeerReview).
In ASK2ME and Waterbird, spurious feature corresponds to an attribute of the input (\texttt{breast\_cancer} for ASK2ME and \texttt{background} for Waterbird).

For each dataset, we consider multiple tasks and study the transfer between these tasks.  Specifically, for each task, we split its data into four environments: $E_1^{\text{train}}, E_2^{\text{train}}, E^{\text{val}}, E^{\text{test}}$, where spurious correlations vary across the two training environments $E_1^{\text{train}}, E_2^{\text{train}}$. 
For the source task $S$, the model can access both of its training environments $E_1^{\text{train}}(S), E_2^{\text{train}}(S)$.
For the target task $T$, the model only has access to one training environment $E_1^{\text{train}}(T)$. We note that the validation set $E^{\text{val}}(T)$ plays an important role in early-stopping and hyper-parameter tuning, especially when the distribution of the data is different between training and testing~\citep{gulrajani2020search}.  In this work, since we don't have access to multiple training environments on the target task, we assume that the validation data $E^{\text{val}}$ follows the same distribution as the training data $E^{\text{train}}_1$.  Table~\ref{tab:data} summarizes the level of the spurious correlations for different tasks.  Additional details can be found in Appendix~\ref{app:dataset}.

\begin{table*}[t]
  \centering
  \caption{Worst-group and average-group accuracy on CelebA. 
  The source task is to predict \texttt{Eyeglasses} and the target task is to predict \texttt{BlondHair}. We use the attribute \texttt{Young} to define two environments: $E_1=\{\texttt{Young}=0\}$, $E_2=\{\texttt{Young}=1\}$. Both environments are available in the source task. In the target task, we only have access to $E_1$ during training and validation..
  We show the results for the first 3 attributes (alphabetical order). The right-most Average$^*$ column is computed based on the performance across all 38 attributes. See Appendix~\ref{appendix:full} for full results.
  }\label{tab:celeba}
  \small
  \begin{tabular}{clcccccccc}
  \toprule
  &\multirow{2}{*}{\textsc{method}} & \multicolumn{2}{c}{\small\texttt{ArchedEyebrows}} & 
  \multicolumn{2}{c}{\small\texttt{Attractive}} &
  \multicolumn{2}{c}{\small\texttt{BagsUnderEyes}} &
  \multicolumn{2}{c}{Average$^*$} \\
  \cmidrule(lr{0.5em}){3-4}\cmidrule(lr{0.5em}){5-6} \cmidrule(lr{0.5em}){7-8} \cmidrule(lr{0.5em}){9-10}
  & & Worst & Average & Worst & Average & Worst & Average & Worst & Average \\\midrule
&\textsc{erm} & 75.43 & 88.52 & 75.00 & 88.94 & 70.91 & 87.09 & 61.01 & 85.07 \\
\midrule\midrule
\hspace{-1mm}\multirow{9}{*}{\STAB{\rotatebox[origin=c]{90}{\textsc{transfer}\hspace{12mm}}}}\hspace{-1mm}
&$\textsc{reuse}_{\textsc{pi}}$& 53.71 & 64.05 & 52.13 & 64.85 & 52.50 & 66.88 & 47.58 & 64.14 \\\cmidrule{2-10}
&$\textsc{reuse}_{\textsc{dann}}$ & 59.56 & 72.44 & 62.03 & 72.26 & 64.58 & 73.83 & 55.27 & 72.31 \\\cmidrule{2-10}
&$\textsc{reuse}_{\textsc{c-dann}}$ & 56.02 & 67.07 & 57.78 & 67.90 & 57.50 & 68.33 & 53.22 & 68.56 \\\cmidrule{2-10}
&$\textsc{reuse}_{\textsc{mmd}}$ & 48.91 & 59.80 & 48.46 & 61.51 & 58.74 & 63.11 & 50.61 & 61.27 \\\cmidrule{2-10}
&$\textsc{finetune}_{\textsc{pi}}$ & 71.86 & 87.02 & 72.73 & 87.34 & 62.50 & 84.10 & 63.07 & 85.27 \\\cmidrule{2-10}
&$\textsc{finetune}_{\textsc{dann}}$ & 65.38 & 83.89 & 63.35 & 84.98 & 56.86 & 81.34 & 50.60 & 80.49 \\\cmidrule{2-10}
&$\textsc{finetune}_{\textsc{c-dann}}$ & 73.85 & 88.90 & 75.61 & 89.39 & 75.86 & 88.14 & 62.03 & 85.57 \\\cmidrule{2-10}
&$\textsc{finetune}_{\textsc{mmd}}$ & 76.07 & 88.80 & 74.33 & 89.74 & 78.57 & 88.61 & 66.80 & 86.81 \\\cmidrule{2-10}
&\textsc{multitask} & 69.66 & 86.91 & 72.73 & 87.44 & 70.00 & 85.21 & 64.37 & 85.21 \\\midrule\midrule
\hspace{-1mm}\multirow{5}{*}{\STAB{\rotatebox[origin=c]{90}{\textsc{auto-debias}\hspace{2mm}}}}\hspace{-1mm}
&\textsc{eiil} & 64.71 & 85.12 & 64.43 & 85.96 & 66.67 & 83.90 & 57.62 & 83.22 \\\cmidrule{2-10}
&\textsc{george} & 74.73 & 87.89 & 73.66 & 87.70 & 77.78 & 87.97 & 63.34 & 85.04 \\\cmidrule{2-10}
&\textsc{lff} & 45.41 & 60.23 & 47.67 & 60.16 & 42.59 & 60.72 & 42.52 & 62.04 \\\cmidrule{2-10}
&\textsc{m-ada} & 64.61 & 83.33 & 67.33 & 83.59 & 70.34 & 85.34 & 54.55 & 80.77 \\\cmidrule{2-10}
&\textsc{dg-mmld} & 69.51 & 87.38 & 68.42 & 87.50 & 63.41 & 84.78 & 55.69 & 83.51 \\\midrule\midrule
&\textsc{tofu} & \textbf{85.66} & \textbf{91.47} & \textbf{88.30} & \textbf{92.76} & \textbf{90.38} & \textbf{92.41} & \textbf{84.86} & \textbf{91.71} \\
\bottomrule
  \end{tabular}
\end{table*}

\paragraph{Natural environments}
We also consider a practical setting where environments are directly defined by a given attribute of the input, and our goal is to reduce model biases from other latent attributes.
We study CelebA~\citep{liu2015deep} where each input (an image of a human face) is annotated with 40 binary attributes. The source task is to predict the \texttt{Eyeglasses} attribute and the target task is to predict the \texttt{BlondHair} attribute. We use the \texttt{Young} attribute to define two environments: $E_1=\{\texttt{Young}=0\}$ and $E_2=\{\texttt{Young}=1\}$. In the source task, both environments are available. In the target task, we only have access to environment $E_1$ during training and validation. At test time, we evaluate the robustness of our target classifier against other latent attributes. Specifically, for each unknown attribute such as \texttt{Male}, we partition the testing data into four groups: $\{\texttt{Male}=1, \texttt{BlondHair}=0\}$, $\{\texttt{Male}=0, \texttt{BlondHair}=0\}$, $\{\texttt{Male}=1, \texttt{BlondHair}=1\}$, $\{\texttt{Male}=0, \texttt{BlondHair}=1\}$. Following \cite{sagawa2019distributionally}, We report the worst-group accuracy and the average-group accuracy.

\subsection{Baselines}
  We compare our algorithm against the following baselines.
  For fair comparison, all methods share the same representation backbone and hyper-parameter search space.
  Implementation details are available at Appendix~\ref{app:details}.
  
  \paragraph{\textsc{erm} baseline}
   We learn a classifier on the target task from scratch by minimizing the average loss across all examples. Note that this classifier is independent of the source task. Its performance reflects the deviation between the training distribution and the testing distribution of the target task.

  \paragraph{Transfer methods}
    Since the source task contains multiple environments, we can learn a stable model on the source task and transfer it to the target task. We use four algorithms to learn the source task:
    \textsc{dann}~\citep{ganin2016domain}, \textsc{c-dann}~\citep{li2018domain2},
    \textsc{mmd}~\citep{li2018domain},
    \textsc{pi}~\citep{yujia2021predict}.
    We consider three standard methods for transferring the source knowledge:
    
    \textsc{reuse}: We directly transfer the feature extractor of the source model to the target task. The feature extractor is fixed when learning the target classifier.
    
    %We first learn a stable model on the source task by contrasting different source environments~\citep{yujia2021predict}. We then directly transfer its feature extractor to the target task. By keeping this feature extractor fixed when learning the target classifier (a linear layer followed by Softmax activation), we hope the resulting model will not rely on these unstable features.
    
    \textsc{finetune}: We update the feature extractor when training the target classifier. \citep{Shafahi2020Adversarially} has shown that \textsc{finetune} may improve adversarial robustness of the target task.
  
    \textsc{multitask}: We adopt the standard multi-task learning approach~\citep{caruana1997multitask} where the source model and the target model share the same feature extractor and are jointly trained together.

  \paragraph{Automatic de-biasing methods}
    For the target task, we can also apply de-biasing approaches that do not require environments. We consider the following baselines:
    
    \textsc{eiil}~\citep{creager2021environment}: Based on a pre-trained ERM classifier's prediction logits, we infer an environment assignment that maximally violates the invariant learning principle~\citep{arjovsky2019invariant}. We then apply group DRO to minimize the worst-case loss over all inferred environments.

    \textsc{george}~\citep{sohoni2020no}: We use the feature representation of a pre-trained ERM classifier to cluster the training data. 
    We then apply group DRO to minimize the worst-case loss over all clusters.
    
    \textsc{lff}~\citep{nam2020learning}: We train a biased classifier together with a de-biased classifier. The biased classifier amplifies its bias by minimizing the generalized cross entropy loss. The de-biased classifier then up-weights examples that are mis-labeled by the biased classifier during training.
    
    \textsc{m-ada}~\citep{qiao2020learning}: We use a Wasserstein auto-encoder to generate adversarial examples. The de-biased classifier is trained on both the original examples and the generated examples.
    
    \textsc{dg-mmld}~\citep{matsuura2020domain}: We iteratively divide target examples into latent domains via clustering, and train the domain-invariant feature extractor via adversarial learning.

  \paragraph{\textsc{oracle}}
  For synthetic environments, we can use the spurious feature to define groups and train an oracle model. For example, in task \textsc{seabird}, this oracle model will minimize the worst-case risk over the following four groups:
  \{\texttt{seabird} in \texttt{water}\}, \{\texttt{seabird} in \texttt{land}\}
  \{\texttt{landbird} in \texttt{water}\}, \{\texttt{landbird} in \texttt{land}\}.
  This oracle model helps us analyze the performance of our proposed algorithm separately from the inherent limitations (such as model capacity and data size).

\begin{figure}[t]
        \centering
        \includegraphics[width=0.6\linewidth]{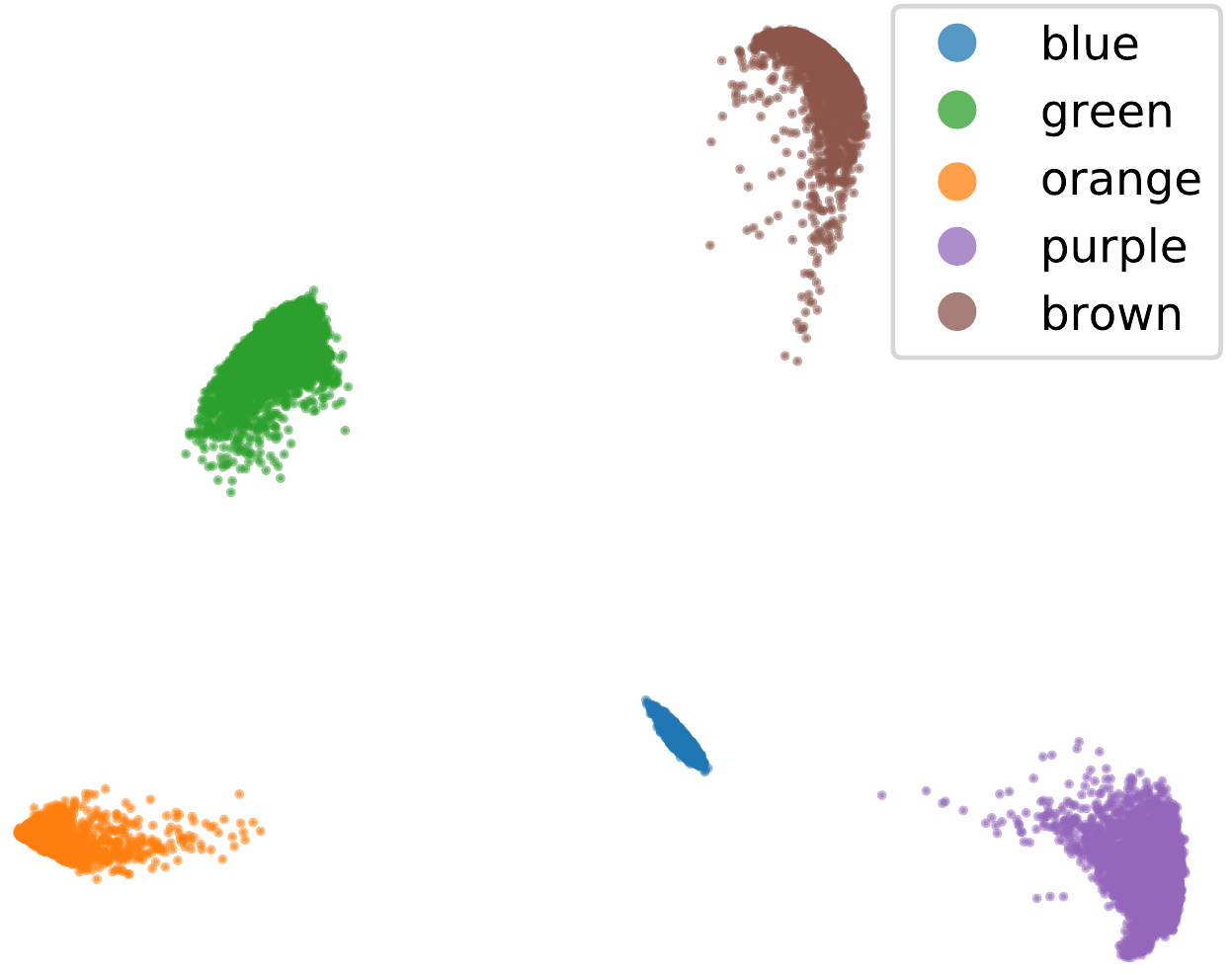}
        \caption{PCA visualization of the unstable feature representation $f_\mathcal{Z}$ for examples in \textsc{mnist} \textsc{even}. $f_\mathcal{Z}$ is trained on \textsc{mnist} \textsc{odd}. \textsc{tofu} identifies the hidden spurious color feature by contrasting different source environments.}
        \label{fig:featurepca}
\end{figure}

\begin{figure}[t]
        \centering
        \includegraphics[width=\linewidth]{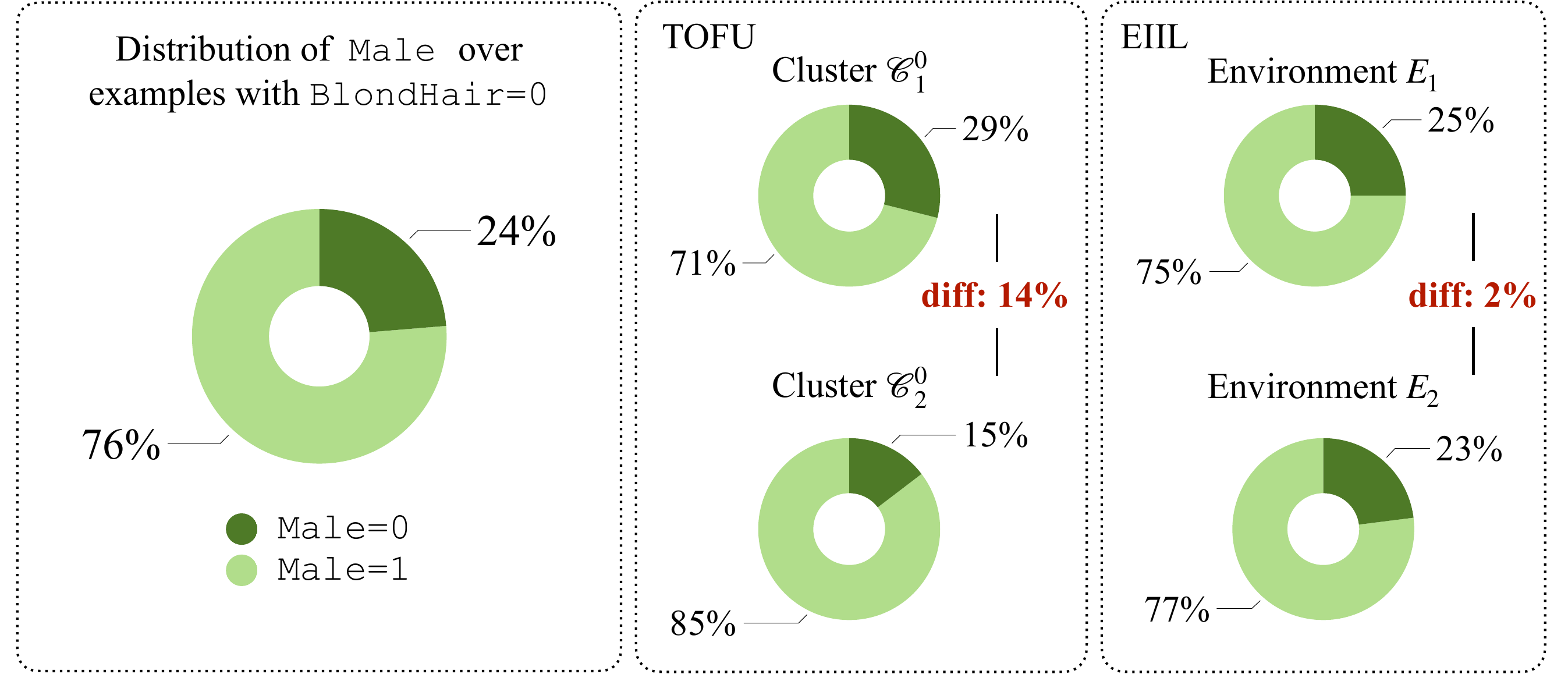}
        \caption{Visualization of the unknown attribute \texttt{Male} on CelebA. Left:
        distributions of \texttt{Male} in the training data. Mid: partitions learned by \textsc{tofu}. Right: partitions learned by \textsc{eiil}. \textsc{tofu} generates partitions that are more informative of the unknown attribute (14\% vs. 2\%). %For comparison, the random splitting baseline has 0\% gap.
        See Appendix~\ref{appendix:full} for results on more attributes.
        }
        \label{fig:visualceleba}
\end{figure}

\begin{table}[t]
  \centering 
  \caption{Quantitative evaluation of the generated clusters against the ground truth unstable features. {For comparison, the \textsc{triplet} baseline (\textsc{trp}) directly encourages source examples with the same label to stay close to each other in the feature space, from which we generate the clusters.} For both methods, we generate two clusters for each target label value and report the average performance across all label values. We observe that the \textsc{triplet} representation, while biased by the spurious correlations, fails to recover the ground truth unstable features for some tasks. By explicitly contrasting the source environments, \textsc{tofu} derives clusters that are highly-informative of the unstable features.}
  \label{tab:cluster}
  \small
  \begin{tabular}{cccccccc}
  \toprule
  \multirow{2}{*}{\hspace{-3mm}\textsc{source}\hspace{-3mm}} & \multirow{2}{*}{\hspace{-3mm}\textsc{target}\hspace{-3mm}} & \multicolumn{2}{c}{Homo.} & \multicolumn{2}{c}{Complete.} & \multicolumn{2}{c}{V-measure} \\ \cmidrule(lr{0.5em}){3-4}\cmidrule(lr{0.5em}){5-6}\cmidrule(lr{0.5em}){7-8}
  & & \textsc{trp} & \hspace{-2mm}\textsc{tofu}\hspace{-2mm} & \textsc{trp} & \hspace{-2mm}\textsc{tofu}\hspace{-2mm}  & \textsc{trp} & \hspace{-2mm}\textsc{tofu}\hspace{-2mm}  \\
  \midrule
 \textsc{odd} & \textsc{even} & 0.42 & \textbf{0.68} & 0.58 & \textbf{0.95} & 0.49 & \textbf{0.79}\\ \midrule
 \textsc{even} & \textsc{odd} & 0.67 & \textbf{0.67} & 0.93 & \textbf{0.99} & 0.78 & \textbf{0.80}  \\ \midrule
\textsc{look}    & \hspace{-3mm}\textsc{aroma}\hspace{-3mm}  & 0.33 & \textbf{0.92} & 0.28 & \textbf{0.92} & 0.30 & \textbf{0.92} \\ \midrule
\textsc{look}    & \hspace{-3mm}\textsc{palate}\hspace{-3mm} & 0.33 & \textbf{0.90} & 0.27 & \textbf{0.89} & 0.30 & \textbf{0.90} \\ \midrule
\hspace{-3mm}\textsc{aroma}\hspace{-3mm}   & \textsc{look}   & 0.33 & \textbf{1.00} & 0.28 & \textbf{1.00} & 0.30 & \textbf{1.00} \\ \midrule
\hspace{-3mm}\textsc{aroma}\hspace{-3mm}   & \hspace{-3mm}\textsc{palate}\hspace{-3mm} & 0.82 & \textbf{1.00} & 0.77 & \textbf{1.00} & 0.79 & \textbf{1.00} \\ \midrule
\hspace{-3mm}\textsc{palate}\hspace{-3mm}  & \textsc{look}   & 0.83 & \textbf{0.98} & 0.77 & \textbf{0.98} & 0.80 & \textbf{0.98} \\ \midrule
\hspace{-3mm}\textsc{palate}\hspace{-3mm}  & \hspace{-3mm}\textsc{aroma}\hspace{-3mm}  & 0.82 & \textbf{0.95} & 0.77 & \textbf{0.95} & 0.79 & \textbf{0.95} \\
\bottomrule
  \end{tabular}
\end{table}

\section{Results}\label{sec:results}
Table~\ref{tab:big} summarizes our results on synthetic environments. We observe that standard transfer methods fail to improve over the \textsc{erm} baseline. On the other hand, \textsc{tofu} consistently achieves the best performance across 12 transfer settings, outperforming the best baseline by 22.9\%. While \textsc{tofu} doesn't have access to the unstable features, by inferring them from the source environments, it matches the oracle performance with only 0.30\% absolute difference.

Table~\ref{tab:celeba} presents our results on natural environments. This task is very challenging as there are multiple latent spurious attributes in the training data. We observe that most automatic de-biasing methods underperform the \textsc{erm} baseline. With the help of the source task, \textsc{finetune} and \textsc{multitask} achieve slightly better performance than \textsc{erm}. \textsc{tofu} continues to shine in this real-world setting: achieving the best worst-group and average-group performance.

%Table~\ref{tab:big} and Table~\ref{tab:celeba} summarize our results. We observe that standard transfer methods fail to improve over the \textsc{erm} baseline. On the other hand, \textsc{tofu} consistently achieves the best performance across 12 transfer settings, outperforming the best baseline by 22.9\% in absolute accuracy. Comparing with automatic de-biasing methods, \textsc{tofu} leverages the source task to better pinpoint the bias and reduces the error rate by 27.3\%.

\textbf{\emph{Is \textsc{tofu} able to identify the unstable features?}} Yes. For synthetic environments, we visualize the unstable feature representation produced by $f_\mathcal{Z}$ on \textsc{mnist even}. Figure~\ref{fig:featurepca} demonstrates that while $f_\mathcal{Z}$ only sees source examples (\textsc{odd}) during training, it can distribute target examples based on their unstable color features.

For natural environments, we visualize the distribution of two latent attributes (\texttt{Male} and \texttt{ArchedEyebrows}) over the generated clusters. Figure~\ref{fig:visualceleba} shows that the distribution gap of the unknown attribute \texttt{Male} across the generated partitions is 2\% for \textsc{eiil}, only marginally better than random splitting (0\%). By leveraging information from the source task, \textsc{tofu} learns partitions that are more informative of the unknown attribute (14\%).

%For text classification, we visualize word importance of $f_\mathcal{Z}$ on \textsc{beer look}. Figure~\ref{fig:analysis} shows that $f_\mathcal{Z}$ primarily focuses on the spurious token when generating the representation.

\textbf{\emph{How do the generated clusters compare to the oracle groups?}} 
We quantitatively evaluate the generated clusters based on three metrics: \emph{homogeneity}
(each cluster contain only examples with the same unstable feature value),
\emph{completeness} (examples with the same unstable feature value belong to the same cluster),
and \emph{V-measure} (the harmonic mean of homogeneity and completeness).
From Table~\ref{tab:cluster}, we see that \textsc{tofu} is able to derive clusters that resemble the oracle groups on \textsc{beer review}.
In \textsc{mnist}, since we generate two clusters for each label value and there are five different colors, it is impossible to recover the oracle groups. However, \textsc{tofu} still achieves almost perfect completeness.
% For comparison, we directly apply K-means clustering to the feature representation learned by \textsc{erm} on the source environments. Although this representation is biased by the unstable feature, the resulting clusters fail to uncover the oracle groups.
\section{Conclusion}\label{sec:conclusion}

Reducing model bias is a critical problem for many machine learning applications in the real world.
In this paper, we recognize that bias is a human-defined concept. Without additional knowledge, automatic de-biasing methods cannot effectively distinguish causal features from spurious features.
The main departure of this paper is to identify bias by using related tasks.
We demonstrate that when the source task and target task share the same set of biases, we can effectively transfer this knowledge and improve the robustness of the target model without additional human intervention.
Compared with 15 baselines across 5 datasets, our approach consistently delivers significant performance gain. Quantitative and qualitative analyses confirm that our method is able to identify the hidden biases.
Due to space limitations, we leave further discussions to Appendix~\ref{app:discussion}.

\section*{Acknowledgement}
This paper is dedicated to the memory of our beloved family member Tofu, who filled our lives with so many wuffs and wuvs.\looseness=-1

We want to thank Menghua Wu, Bracha Laufer, Adam Yala, Adam Fisch, Tal Schuster, Yujie Qian, Victor Quach and Jiang Guo for their thoughtful feedback.

Research was sponsored by the United States Air Force Research Laboratory and the United States Air Force Artificial Intelligence Accelerator and was accomplished under Cooperative Agreement Number FA8750-19-2-1000. The views and conclusions contained in this document are those of the authors and should not be interpreted as representing the official policies, either expressed or implied, of the United States Air Force or the U.S. Government. The U.S. Government is authorized to reproduce and distribute reprints for Government purposes notwithstanding any copyright notation herein.

The research, supported by the Defence Science and Technology Agency, was also funded by the Wellcome Trust (grant reference number: 222396/Z/21/Z, \emph{Building Trustworthy Clinical AI: from algorithms to clinical deployment}).

The research was also sponsored by the MIT Jameel Clinic.

\bibliography{ref}
\bibliographystyle{icml2022}

\newpage
\onecolumn

\appendix

\section{Discussion}
\label{app:discussion}
\paragraph{Are biases shared across real-world tasks?}
In this paper, we show that for tasks where the biases are shared, we can effectively transfer this knowledge to obtain a more robust model. This assumption holds in many real world applications. For example, in natural language processing, the same gender bias exist across many tasks including relation extraction~\citep{gaut2020towards}, semantic role labeling~\citep{jia2020mitigating}, abusive language detection~\citep{park2018reducing} and sentiment analysis~\citep{kiritchenko2018examining}. In computer vision, the same geographical bias exists across different object recognition benchmarks such as ImageNet, COCO and OpenImages~\citep{de2019does}.

When a single source task does not describe all unwanted unstable features, we can leverage multiple source tasks and compose their individual unstable features together. We can naturally extend \textsc{tofu} to accomplish this goal by learning the unstable feature representation jointly across this collection of source tasks. We focus on the basic setting in our main paper and leave this extension to Appendix~\ref{app:multi_source}.

Our approach is not applicable in situations where the biases in the source task and the target task are completely disjoint.

\paragraph{What if the source task and target task are from different domains?}
In this paper, we focus on the setting where the source task and the target task are from the same domain. If the target task is drawn from a different domain, we can use domain-adversarial training to align the distributions of the unstable features across the source domain and the target domain~\citep{li2018domain,li2018domain}. Specifically, when training the unstable feature representation $f_z$, we can introduce an adversarial player that tries to guess the domain label from $f_z$. The representation $f_z$ is updated to fool this adversarial player in addition to minimize the triplet loss in Eq~\eqref{eq:hinge}. We leave this extension to future work.

\paragraph{Can we apply domain-invariant representation learning (DIRL) directly to the source environments?}
Domain invariant representation learning~\citep{ganin2016domain, li2018domain2, li2018domain} aims to match the feature representations across domains. If we directly treat environments as domains and apply these methods, the resulting representation may still encode unstable features.

For example, in CelebA, the attribute \texttt{Male} is spuriously correlated with the target attribute \texttt{BlondHair} (Women are more likely to have blond hair than men in this dataset). Given the two environments $\{\texttt{Young}=0\}$ and \{\texttt{Young}=1\}, DIRL learns an age-invariant representation. However, if the distribution of \texttt{Male} is the same across the two environments, DIRL will encode this attribute into the age-invariant representation (since it is helpful for predicting the target \texttt{BlondHair} attribute).
In our approach, we realize that the the correlations between \texttt{Male} and \texttt{BlondHair} are different in the two environments (The elderly may have more white hair). Even though the distribution of \texttt{Male} may be the same, we can still identify this bias from the classifiers’ mistakes.
Empirically, Table~\ref{tab:celeba} shows that while DIRL methods improve over the ERM baseline, they still perform poorly on minority groups (worst case acc 66.80\% on CelebA).

\paragraph{What if the mistakes correspond to other factors such as label noise, distribution shifts, etc.?}
For ease of analysis, we do not consider label noise and distribution shift in Theorem~\ref{thm:1}. One future direction is to model bias from the information perspective (rather than looking at the linear correlations). This will enable us to relax the assumption in the analyses and we can further incorporate these different mistake factors into the modeling.

We note that we do not impose this assumption in our empirical experiments. For example, we explicitly added label noise into the \textsc{mnist} data. In \textsc{celeba}, there is a distribution gap (from young people to the elderly) across the two environments. We observe that our method is able to perform robustly in situations where the assumption breaks.

\paragraph{Is the algorithm efficient when multiple source environments are available?}
Our method can be generalized efficiently to multiple environments. 
Given $N$ source environments, we first note that the complexities of the target steps \textbf{T.1} and \textbf{T.2} are independent of $N$.
For the source task, the $N$ environment-specific classifiers (in \textbf{S.1}) can be learned jointly with multi-task learning~\citep{caruana1997multitask}. This significantly reduces the inference cost at \textbf{S.2} as we only need to pass each input example through the (expensive) representation backbone for one time. In \textbf{S.3}, we sample partitions when minimizing the triplet loss, so there is no additional cost during training. In this paper, we focus on the two-environments setup for simplicity and leave this generalization to future work.

\paragraph{Why does the baselines perform so poorly on \textsc{mnist}?}
We note that the representation backbone (a 2-layer CNN) on MNIST is trained from scratch while we use pre-trained representations for other datasets (see Appendix~\ref{app:details}). Our hypothesis is that models are more prune to spurious correlations when trained from scratch.

\section{Theoretical analysis}\label{app:theory}

\subsection{Partitions reveal the unstable correlation}
We start by reviewing the results in  \cite{yujia2021predict} which shows that the generated partitions reveal the unstable correlation. We consider binary classification tasks where $\mathcal{Y} \in \{0, 1\}$. For a given input $x$, we use $\mathcal{C}(x)$ to represent its stable (causal) feature and $\mathcal{Z}(x)$ to represent its unstable feature.
In order to ease the notation, if no confusion arises, we omit the dependency on $x$.
We use lowercase letters $c, z, y$ to denote the specific values of $\mathcal{C}, \mathcal{Z}, \mathcal{Y}$.
\begin{prop}
  \label{prop:1}
  For a pair of environments $E_i$ and $E_j$,
  assuming that the classifier $f_i$ is able to learn the true conditional
  $P_i(\mathcal{Y}\mid C, \mathcal{Z})$,
  we can write the joint distribution $P_j$ of $E_j$ as the mixture of
  $P_j^{i \checkmark}$ and $P_j^{i \times}$:
  \[
    P_j(c, z, y) = \alpha_j^i P_j^{i \checkmark} (c, z, y) +
    (1-\alpha_j^i)  P_j^{i \times} (c, z, y),
  \]
  where $\alpha_j^i = \sum_{c, z, y} P_j(c, z, y)\cdot P_i(y \mid c, z)$ and
  \[
    \begin{aligned}
      P_j^{i\checkmark}(c, z, y) &\propto
      P_j(c, z, y)\cdot P_i(y \mid c, z),\\
      P_j^{i\times}(x, z, y) &\propto
      P_j(c, z, y)\cdot P_i(1-y \mid c, z).
    \end{aligned}
  \]
\end{prop}
\begin{proof}
See \cite{yujia2021predict}.
\end{proof}
Proposition~\ref{prop:1} tells us that if $f_i$ is powerful enough to capture the true conditional in $E_i$, partitioning the environment $E_j$ is equivalent to scaling its joint distribution based on the conditional on $E_i$.

Now suppose that the marginal distribution of
$\mathcal{Y}$ is uniform in all joint distributions,
i.e., $f_i$ performs equally well on different labels. \citet{yujia2021predict} shows that the unstable correlations will have different signs in the subset of correct predictions and in the subset of incorrect predictions.

\begin{prop}\label{prop:2}
  Suppose $\mathcal{Z}$ is independent of $\mathcal{C}$ given $\mathcal{Y}$.
  For any environment pair $E_i$ and $E_j$,
  if $\sum_y P_i(z \mid y) = \sum_y P_j(z \mid y)$ for any $z$,
  then $\mathrm{Cov}(\mathcal{Z}, \mathcal{Y}; P_i) > \mathrm{Cov}(\mathcal{Z}, \mathcal{Y}; P_j)$ implies
  \[
  \mathrm{Cov}(\mathcal{Z}, \mathcal{Y}; P_j^{i\times}) < 0, 
  \quad\text{and}\quad
  \mathrm{Cov}(\mathcal{Z}, \mathcal{Y}; P_i^{j\times}) > 0.
  \]
\end{prop}
\begin{proof}
See \cite{yujia2021predict}.
\end{proof}
Proposition~\ref{prop:2} implies that no matter whether the spurious correlation is positive or negative, by interpolating $P_j^{i\checkmark}, P_j^{i\times},P_i^{j\checkmark}, P_i^{j\times}$,
we can obtain an \emph{oracle} distribution where the spurious correlation between $\mathcal{Z}$ and $\mathcal{Y}$ vanishes. Since the oracle interpolation coefficients are not available in practice, \citet{yujia2021predict} propose to optimize the worst-case risk across all interpolations of the partitions.

\subsection{Partitions reveal the unstable feature}
Proposition~\ref{prop:2} shows that the partitions $E_j^{i\checkmark}, E_j^{i\times},E_i^{j\checkmark}, E_i^{j\times}$ are informative of the biases. However these partitions are not transferable as they are coupled with task-specific information, i.e., the label $\mathcal{Y}$. To untangle this dependency, we look at different label values and obtain the following result.
\begin{cor}
  \label{cor:1}
  Under the same assumption as Proposition~\ref{prop:2}, if $\mathrm{Cov}(\mathcal{Z}, \mathcal{Y}; P_i) > \mathrm{Cov}(\mathcal{Z}, \mathcal{Y}; P_j) > 0$ and $\mathcal{Z}$ follows a uniform distribution within each partition, then
  \begin{align*}
    \sum_{z} z P_j^{i\times} (\mathcal{Z}=z, \mathcal{Y}=1) &> \sum_{z} z P_j^{i\checkmark} (\mathcal{Z}=z, \mathcal{Y}=1),\\
    \sum_{z} z P_j^{i\times} (\mathcal{Z}=z, \mathcal{Y}=0) &< \sum_{z} z P_j^{i\checkmark} (\mathcal{Z}=z, \mathcal{Y}=0).\\
  \end{align*}
\end{cor}
\begin{proof}
By definition of the covariance, we have
\[
\mathrm{Cov}(\mathcal{Z}, \mathcal{Y}) = 
\sum_{z, y} z y P(\mathcal{Z}=z, \mathcal{Y}=y) - \left(\sum_{z} zP(\mathcal{Z}=z)\right)
\left( \sum_{y}y P(\mathcal{Y}=y)\right)
\]
Since we assume the marginal distribution of the label is uniform, we have $\sum_y y P(\mathcal{Y}=y) = 0.5$. Then we have
\[
\mathrm{Cov}(\mathcal{Z}, \mathcal{Y}) = 
\sum_{z} z P(\mathcal{Z}=z, \mathcal{Y}=1) - 0.5 \sum_{z} zP(\mathcal{Z}=z).
\]
Using $P(\mathcal{Z}=z) = P(\mathcal{Z}=z, \mathcal{Y}=0) + P(\mathcal{Z}=z, \mathcal{Y}=1)$, we obtain
\begin{equation}\label{eq:cov_general}
\mathrm{Cov}(\mathcal{Z}, \mathcal{Y}) = 
0.5\sum_{z} z P(\mathcal{Z}=z, \mathcal{Y}=1) - 0.5 \sum_{z} zP(\mathcal{Z}=z, \mathcal{Y}=0).
\end{equation}
From Proposition~\ref{prop:2}, we have
$\mathrm{Cov}(\mathcal{Z}, \mathcal{Y}; P_j^{i\times}) < 0$. 
Note that this implies $\mathrm{Cov}(\mathcal{Z}, \mathcal{Y}; P_j^{i\checkmark}) > 0$ since $\mathrm{Cov}(\mathcal{Z}, \mathcal{Y}; P_j) > 0$ and $P_j = \alpha_j^i P_j^{i\checkmark} + (1-\alpha_j^i) P_j^{i\times}$.
Combining with Eq~\eqref{eq:cov_general}, we have
\begin{align}\label{eq:ineq}
\sum_{z} z P_j^{i\times}(\mathcal{Z}=z, \mathcal{Y}=1) &< \sum_{z} zP_j^{i\times}(\mathcal{Z}=z, \mathcal{Y}=0),\nonumber\\
\sum_{z} z P_j^{i\checkmark}(\mathcal{Z}=z, \mathcal{Y}=1) &> \sum_{z} zP_j^{i\checkmark}(\mathcal{Z}=z, \mathcal{Y}=0).
\end{align}
Since we assume the marginal distribution of the unstable feature $\mathcal{Z}$ is uniform, we have
\begin{align}\label{eq:sum}
\sum_{z} z P_j^{i\times}(\mathcal{Z}=z, \mathcal{Y}=1) + \sum_{z} zP_j^{i\times}(\mathcal{Z}=z, \mathcal{Y}=0) = \sum_{z} z P_j^{i\times}(\mathcal{Z}=z) = 0.5,\nonumber\\
\sum_{z} z P_j^{i\checkmark}(\mathcal{Z}=z, \mathcal{Y}=1) + \sum_{z} zP_j^{i\checkmark}(\mathcal{Z}=z, \mathcal{Y}=0) = \sum_{z} z P_j^{i\checkmark}(\mathcal{Z}=z) = 0.5.
\end{align}
Plugging Eq~\eqref{eq:sum} into Eq~\eqref{eq:sum}, we have
\begin{align*}
\sum_{z} z P_j^{i\times}(\mathcal{Z}=z, \mathcal{Y}=1) &< 0.25 < \sum_{z} zP_j^{i\times}(\mathcal{Z}=z, \mathcal{Y}=0),\nonumber\\
\sum_{z} z P_j^{i\checkmark}(\mathcal{Z}=z, \mathcal{Y}=1) &> 0.25 > \sum_{z} zP_j^{i\checkmark}(\mathcal{Z}=z, \mathcal{Y}=0).
\end{align*}
Combining the two inequalities finishes the proof.
\end{proof}

Corollary~\ref{cor:1} shows that if we look at examples within the same label value, then expectation of the unstable feature $\mathcal{Z}$ within the set of correct predictions will diverge from the one within the set of incorrect predictions.
In order to learn a metric space that corresponds to the values of $\mathcal{Z}$, we sample different batches from the partitions and prove the following theorem.

\addtocounter{thm}{-1}
\begin{thm}(Full version)
Suppose $\mathcal{Z}$ is independent of $\mathcal{C}$ given $\mathcal{Y}$.
%For any environment pair $E_i$ and $E_j$, $\sum_y P_i(z \mid y) = \sum_y P_j(z \mid y)$ for any $z$.
We assume that $\mathcal{Y}$ and $\mathcal{Z}$ both follow a uniform distribution within each partition.
  
Consider examples in $E_j$ with label value $y$. Let $X_1^\checkmark, X_2^\checkmark$ denote two batches of examples that $f_i$ predicted correctly, and let $X_3^\times$ denote a batch of incorrect predictions. If $\mathrm{Cov}(\mathcal{Z}, \mathcal{Y}; P_i) > \mathrm{Cov}(\mathcal{Z}, \mathcal{Y}; P_j) > 0$, we have
\[
\| \overline{\mathcal{\mathcal{Z}}}(X_1^\checkmark) - \overline{\mathcal{\mathcal{Z}}}(X_2^\checkmark) \|_2
< \| \overline{\mathcal{\mathcal{Z}}}(X_1^\checkmark) - \overline{\mathcal{\mathcal{Z}}}(X_3^\times) \|_2
\]
almost surely for large enough batch size.
\end{thm}
\begin{proof}
Without loss of generality, we consider $y=0$.
Let $n$ denote the batch size of $X_1^\checkmark$, $X_2^\checkmark$ and $X_3^\checkmark$.
By the law of large numbers, we have
\[
\overline{\mathcal{Z}}(X_1^\checkmark), \overline{\mathcal{Z}}(X_2^\checkmark) \xrightarrow{\text{a.s.}}  \mathbb{E}_{P_j^{i\checkmark}(\mathcal{Z} \mid \mathcal{Y})} \left[\mathcal{Z} \mid \mathcal{Y}=0\right]\quad\text{and}\quad
\overline{\mathcal{Z}}(X_3^\times) \xrightarrow{\text{a.s.}} \mathbb{E}_{P_j^{i\times}(\mathcal{Z} \mid \mathcal{Y})} \left[\mathcal{Z} \mid \mathcal{Y}=0\right],
\]
as $n \rightarrow \infty$. Note that Corollary~\ref{cor:1} tells us
\[
\mathbb{E}_{P_j^{i\times}(\mathcal{Z} \mid \mathcal{Y})} \left[\mathcal{Z} \mid \mathcal{Y}=0\right] < \mathbb{E}_{P_j^{i\checkmark}(\mathcal{Z} \mid \mathcal{Y})} \left[\mathcal{Z} \mid \mathcal{Y}=0\right].
\]
Thus we have 
\[
\| \overline{\mathcal{\mathcal{Z}}}(X_1^\checkmark) - \overline{\mathcal{\mathcal{Z}}}(X_2^\checkmark) \|_2
< \| \overline{\mathcal{\mathcal{Z}}}(X_1^\checkmark) - \overline{\mathcal{\mathcal{Z}}}(X_3^\times) \|_2
\]
almost surely as $n\rightarrow \infty$.
\end{proof}
We note that while we focus our theoretical analysis on binary tasks, empirically, our method is able to correctly identify the hidden bias for multi-dimensional unstable features and multi-dimensional label values.
\section{Experimental setup}
\subsection{Datasets and models}
\label{app:dataset}

\subsubsection{MNIST}
\paragraph{Data}
We extend \citet{arjovsky2019invariant}'s approach for generating spurious correlations and define two \emph{multi-class} classification tasks: \textsc{even} (5-way classification among digits \texttt{0,2,4,6,8}) and \textsc{odd} (5-way classification among digits \texttt{1,3,5,7,9}).
For each image, we first map its numeric digit value $y^{\text{digit}}$ into its class id within the task: $y^\text{causal} = \lfloor y^{\text{digit}} / 2\rfloor$. This class id serves as the causal feature for the given task.
We then sample the observed label $y$, which equals to $y^\text{causal}$ with probability 0.75 and a uniformly random other label value with the remaining probability.
With this noisy label, we now sample the spurious color feature:
the color value equals $y$ with $\eta$ probability and a uniformly other value with the remaining probability.
We note that since there are five different digits for each task, we have five different colors.
Finally, we color the image according to the generated color value.
For the training environments, we set $\eta$ to $0.8$ in $E_1^{\text{train}}$ and $0.9$ in $E_2^{\text{train}}$.
We set $\eta=0.1$ in the testing environment $E^{\text{test}}$.

We use the official train-test split of MNIST. Training environments are constructed from training split, with 7370 examples per environment for \textsc{even} and 7625 examples per environment for \textsc{odd}. Validation data and testing data is constructed based on the testing split. For \textsc{even}, both validation data and testing data have 1230 examples. For \textsc{odd}, the number is 1267. Following \citet{arjovsky2019invariant}, We convert each grey scale image into a $5\times 28 \times 28$ tensor, where the first dimension corresponds to the spurious color feature.

\paragraph{Representation backbone}
We follow the architecture from PyTorch's MNIST example\footnote{https://github.com/pytorch/examples/blob/master/mnist/main.py}.
Specifically, each input image is passed to a CNN with 2 convolution layers followed by 2 fully connected layers.

\paragraph{License}
The dataset is freely available at ~\url{http://yann.lecun.com/exdb/mnist/}.

\subsubsection{Beer Review}
\paragraph{Data}
We consider the transfer among three \emph{binary} aspect-level sentiment classification tasks: \textsc{look}, \textsc{aroma} and \textsc{palate}~\citep{lei2016rationalizing}.
For each review, we follow \citet{yujia2021predict} and append a pseudo token (\texttt{art\_pos} or \texttt{art\_neg}) based on the the sentiment of the given aspect (\texttt{pos} or \texttt{neg}). 
The probability that this pseudo token agrees with the sentiment label is $0.8$ in $E_1^{\text{train}}$ and $0.9$ in $E_2^{\text{train}}$. In the testing environment, this probability reduces to $0.1$.
Unlike MNIST, there is no label noise added to the data.

We use the script created by \citet{yujia2021predict} to generate spurious features for each aspect.
Specifically, for each aspect, we randomly sample training/validation/testing data from the dataset.
Since our focus in this paper is to measure whether the algorithm is able to remove biases (rather than label imbalance), we maintain the marginal distribution of the label to be uniform.
Each training environment contains 4998 examples.
The validation data contains 4998 examples and the testing data contains 5000 examples.
The vocabulary sizes for the three aspects (look, aroma, palate) are: 10218, 10154 and 10086.

\paragraph{Representation backbone}
We use a 1D CNN~\cite{kim-2014-convolutional}, with filter size $3, 4, 5$, to obtain the feature representation.
Specifically, each input is first encoded by pre-trained FastText embeddings~\cite{mikolov2018advances}.
Then it is passed into a convolution layer followed by max pooling and ReLU activation.

\paragraph{License}
This dataset was originally downloaded from \url{https://snap.stanford.edu/data/web-BeerAdvocate.html}.
As per request from BeerAdvocate the data is no longer publicly available.

\subsubsection{ASK2ME}
\paragraph{Data}
ASK2ME~\citep{doi:10.1200/CCI.19.00042} is a text classification dataset where the inputs are paper abstracts from PubMed. We study the transfer between two \emph{binary} classification tasks: \textsc{penetrance} (identifying whether the abstract is informative about the risk of cancer for gene mutation carriers) and \textsc{incidence} (identifying whether the abstract is informative about proportion of gene mutation carriers in the general population).
By definition, both tasks are causally-independent of the diseases that have been studied in the abstract. However, due to the bias in the data collection process, \citet{deng2019validation} found that the performance varies (by 12\%) when we evaluate based on different cancers. To assess whether we can remove such bias, we define two training environments for each task based on the correlations between the task label and the \texttt{breast\_cancer} attribute (indicating the presence of breast cancer in the abstract). Script for generating the environments is available in the supplemental materials. Note that the model doesn't have access to the \texttt{breast\_cancer} attribute during training.

Following~\citet{sagawa2019distributionally}, we evaluate the performance on a balanced test environment where there is no spurious correlation between \texttt{breast\_cancer} and the task label. This helps us understand the overall generalization performance across different input distributions. %\shiyu{Actually, we may want to say about, both tasks could be served as the source or the target. If it served as source, we will use the partition of $E^{\text{train}}_1$, $E^{\text{train}}_2$, and $E^{\text{val}}$.  And we it used as the target, we use $E^{\text{train}}_1$, and $E^{\text{val}}$ for training and test on the $E^{\text{test}}$ partition. }

We randomly split the data and use 50\% for \textsc{penetrance} and 50\% for \textsc{incidence}.
For \textsc{penetrance}, there are 948 examples in $E_1^{\text{train}}$ and $E^{\text{val}}$, 816 examples in $E_2^{\text{train}}$ and 268 examples in $E^{\text{test}}$.
For \textsc{incidence}, there are 879 examples in $E_1^{\text{train}}$ and $E^{\text{val}}$, 773 examples in $E_2^{\text{train}}$ and 548 examples in $E^{\text{test}}$.
The processed data will be publicly available.

\paragraph{Representationi backbone}
The model architecture is the same as the one for Beer review.

\paragraph{License}
MIT License.

\subsubsection{Waterbird}
\paragraph{Data}
Waterbird is an image classification dataset where each image is labeled based on its bird class~\citep{WelinderEtal2010} and the background attribute (\texttt{water} vs.\ \texttt{land}).
Following~\citet{sagawa2019distributionally}, we group different bird classes together and consider two \emph{binary} classification tasks: \textsc{seabird} (classifying 36 seabirds against 36 landbirds) and \textsc{waterfowl} (classifying 9 waterfowl against 9 \emph{different} landbirds).
Similar to ASK2ME, we define two training environments for each task based on the correlations between the task label and the \texttt{background} attribute. Script for generating the environments is available in the supplemental materials. At test time, we measure the generalization performance on a balanced test environment. 

Following \citet{liu2015faceattributes}, we group different classes of birds together to form  binary classification tasks. 

In \textsc{waterfowl}, the task is to identify 9 different waterfowls (Red breasted Merganser, Pigeon Guillemot, Horned Grebe, Eared Grebe, Mallard, Western Grebe, Gadwall, Hooded Merganser, Pied billed Grebe) against 9 different landbirds (Mourning Warbler, Whip poor Will, Brewer Blackbird, Tennessee Warbler, Winter Wren, Loggerhead Shrike, Blue winged Warbler, White crowned Sparrow, Yellow bellied Flycatche).
The training environment $E_1^{\text{train}}$ contains 298 examples and the training environment $E_2^{\text{train}}$ contains 250 examples. The validation set has $300$ examples and the test set has 216 examples. 

In \textsc{seabird}, the task is to identify \emph{36 different seabirds} (Heermann Gull, Red legged Kittiwake, Rhinoceros Auklet, White Pelican, Parakeet Auklet, Western Gull, Slaty backed Gull, Frigatebird, Western Meadowlark, Long tailed Jaeger, Red faced Cormorant, Pelagic Cormorant, Brandt Cormorant, Black footed Albatross, Western Wood Pewee, Forsters Tern, Glaucous winged Gull, Pomarine Jaeger, Sooty Albatross, Artic Tern, California Gull, Horned Puffin, Crested Auklet, Elegant Tern, Common Tern, Least Auklet, Northern Fulmar, Ring billed Gull, Ivory Gull, Laysan Albatross, Least Tern, Black Tern, Caspian Tern, Brown Pelican, Herring Gull, Eastern Towhee) against \emph{36 different landbirds} (Prairie Warbler, Ringed Kingfisher, Warbling Vireo, American Goldfinch, Black and white Warbler, Marsh Wren, Acadian Flycatcher, Philadelphia Vireo, Henslow Sparrow, Scissor tailed Flycatcher, Evening Grosbeak, Green Violetear, Indigo Bunting, Gray Catbird, House Sparrow, Black capped Vireo, Yellow Warbler, Common Raven, Pine Warbler, Vesper Sparrow, Pileated Woodpecker, Bohemian Waxwing, Bronzed Cowbird, American Three toed Woodpecker, Northern Waterthrush, White breasted Kingfisher, Olive sided Flycatcher, Song Sparrow, Le Conte Sparrow, Geococcyx, Blue Grosbeak, Red cockaded Woodpecker, Green tailed Towhee, Sayornis, Field Sparrow, Worm eating Warbler).
The training environment $E_1^{\text{train}}$ contains 1176 examples and the training environment $E_2^{\text{train}}$ contains 998 examples. The validation set has 1179 examples and the test set has 844 examples.

\paragraph{Representation backbone}
We use the Pytorch torchvision implementation of the ResNet50 model, starting from pretrained weights.
We re-initalize the final layer to predict the label.

\paragraph{License}
This dataset is publicly available at \url{https://nlp.stanford.edu/data/dro/waterbird_complete95_forest2water2.tar.gz}

\subsubsection{CelebA}
\paragraph{Data}
CelebA~\citep{liu2015deep} is an image classification dataset where each image is annotated with 40 binary attributes. We consider \texttt{Eyeglasses} as the source task and \texttt{BlondHair} as the target task. We split the official train / val / test set into two parts (uniformly at random) for each task. We use the attribute \texttt{Young} to create two environments: $E_1=\{\texttt{Young}=0\}$, $E_2=\{\texttt{Young}=1\}$. For the target task, the model only has access to $E_1$ during training and validation. Table~\ref{tab:celeba_stat} summarizes the data statistics.

\begin{table}[t]
\centering
\caption{Data statistics of CelebA. The model has access to both $E_1$ and $E_2$ on the source task. For the target task, only $E_1$ is available during training and validation.}
\label{tab:celeba_stat}
\begin{tabular}{lcccc}
\toprule
      & \multicolumn{2}{c}{source task \texttt{Eyeglasses}} & \multicolumn{2}{c}{target task \texttt{BlondHair}} \\
      \cmidrule{2-3}\cmidrule{4-5}
      & $E_1:\{\texttt{Young=0}\}$ & $E_2:\{\texttt{Young=1}\}$ & $E_1:\{\texttt{Young=0}\}$ & $E_2:\{\texttt{Young=1}\}$ \\\midrule
Train & 17955 & 63430 & 17973 & 63412    \\\midrule
Val   & 2494  & 7442  & 2453   & 7480    \\\midrule
Test  & 2452  & 7597  & 2444   & 7537    \\\bottomrule
\end{tabular}
\end{table}

\paragraph{License}
The CelebA dataset is available for non-commercial research purposes only. It is publicly available at \url{https://mmlab.ie.cuhk.edu.hk/projects/CelebA.html}.

\paragraph{Representation backbone}
We use the Pytorch torchvision implementation of the ResNet50 model, starting from pretrained weights.
We re-initalize the final layer to predict the label.

\subsection{Implementation details}
\label{app:details}
\paragraph{For all methods:}
We use batch size $50$ and evaluate the validation performance every $100$ batch.
We apply early stopping once the validation performance hasn't improved in the past 20 evaluations.
We use Adam~\cite{kingma2014adam} to optimize the parameters and tune the learning rate $\in\{10^{-3}, 10^{-4}\}$.
For simplicity, we train all methods without data augmentation.
Following \citet{sagawa2019distributionally}, we apply strong regularizations to avoid over-fitting.
Specifically, we tune the dropout rate $\in\{0.1, 0.3, 0.5\}$ for text classification datasets (Beer review and ASK2ME) and tune the weight decay parameters $\in\{10^{-1}, 10^{-2}, 10^{-3}\}$ for image datasets (MNIST, Waterbird and CelebA).

\textbf{\textsc{dann}, \textsc{c-dann}}
For the domain adversarial network, we use a MLP with 2 hidden ReLU layer with 300 neurons for each layer. The representation backbone is updated via a gradient reversal layer. We tune the weight of the adversarial loss $\in\{0.01, 0.1, 1\}$. 

\textbf{\textsc{mmd}} We match the mean and covariance of the features across the two source environments. We use the implementation from \url{https://github.com/facebookresearch/DomainBed/blob/main/domainbed/algorithms.py}. We tune the weight of the MMD loss $\in\{0.01, 0.1, 1\}$.

\textbf{\textsc{multitask}} For the source task, we first partition the source data into subsets with opposite spurious correlations~\citep{yujia2021predict}. During multi-task training, we minimize the worst-case risk over all these subsets for the source task and minimize the average empirical risk for the target task. \textsc{multitask} is more flexible than \textsc{reuse} since we tune feature extractor to fit the target data. Compared to \textsc{finetune}, \textsc{multitask} is more constrained as the source model prevents over-utilization of unstable features during joint training.

\textbf{Ours}
We fix $\delta=0.3$ in all our experiments.
Based on our preliminary experiments (Figure~\ref{fig:cluster}), we fix the number of clusters to be $2$ for all our experiments in Table~\ref{tab:big} and Table~\ref{tab:celeba}.
For the target classifier, we directly optimize the $\min-\max$ objective. Specifically, at each step, we sample a batch of example from each group, and minimize the worst-group loss. We found the training process to be pretty stable when using the Adam optimizer.

\textbf{Validation criteria}
For \textsc{erm}, \textsc{reuse}, \textsc{finetune} and \textsc{multitask}, since we don't have any additional information (such as environments) for the target data, we apply early stopping and hyper-parameter selection based on the average accuracy on the validation data.

For \textsc{tofu}, since we have already learned an unstable feature representation $f_{\mathcal{Z}}$ on the source task, we can also use it to cluster the validation data into groups where the unstable features within each group are different. We measure the worst-group accuracy and use it as our validation criteria.

For \textsc{oracle}, as we assume access to the oracle unstable features for the target data, we can use them to define groups on the validation data as well. We use the worst-group accuracy as our validation criteria.

We also note that when we transfer from \textsc{look} to \textsc{aroma} in Table~\ref{tab:big}, both \textsc{tofu} and \textsc{oracle} are able to achieve $~75$ accuracy on $E^{\text{test}}$. 
This number is higher than the performance of training on \textsc{aroma} with two data environments (~68 accuracy in Table~\ref{tab:big}).
This result makes sense since in the latter case, we only have in-domain validation set and we use the average accuracy as our hyper-parameter selection metric. However, in both \textsc{tofu} and \textsc{oracle}, we create (either automatically or manually) groups over the validation data and measure the worst-group performance. This ensures that the chosen model will not over-fit to the unstable correlations.

\paragraph{Computational resources:}
We use our internal clusters (24 NVIDIA RTX A6000 and 16 Tesla V100-PCIE-32GB) for the experiments. It took around a week to generate all the results in Table~\ref{tab:big} and Table~\ref{tab:celeba}.
\section{Additional analysis}
\textbf{\emph{Why do the baselines behave so differently across different datasets?}}
As \citet{bao2019few} pointed out, the transferability of the low-level features is very different in image classification and in text classification. For example, the keywords for identifying the sentiment of \textsc{look} are very different from the ones for \textsc{palate}. Thus, fine-tuning the feature extractor is crucial. This explains why \textsc{reuse} underperforms other baselines on text data.
Conversely, in image classification, the low-level patterns (such as edges) are more transferable across tasks. Directly reusing the feature extractor helps improve model stability against spurious correlations.
Finally, we note that since \textsc{tofu} transfers the unstable features instead of the task-specific causal features, it performs robustly across all the settings.

\begin{figure}[t]
  \begin{center}
    \includegraphics[width=0.4\linewidth]{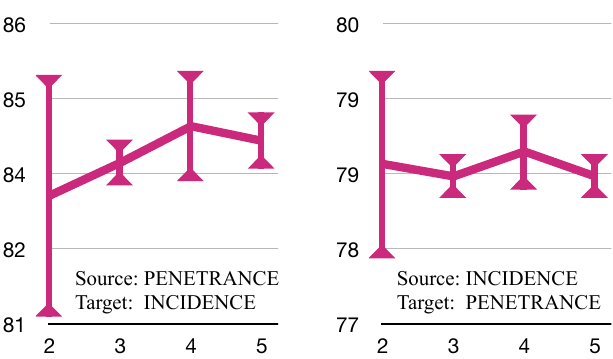}
  \end{center}
  \caption{Accuracy of \textsc{tofu} on \textsc{ask2me} as we vary the number of clusters $n_c$ generated for each label value. Empirically, we see that while having more clusters doesn't improve the performance, it helps reduce the variance.}
  \label{fig:cluster}
\end{figure}

\textbf{\emph{How many clusters to generate?}}
We study the effect of the number of clusters on \textsc{ask2me}.
Figure~\ref{fig:cluster} shows that while generating more clusters in the unstable feature space $f_\mathcal{Z}$ reduces the variance, it doesn't improve the performance by much.
This is not very surprising as the training data is primarily biased by a single \texttt{breast\_cancer} attribute.
We expect that having more clusters will be beneficial for tasks with more sophisticated underlying biases.

\begin{figure}[t]
  \begin{center}
    \includegraphics[width=0.6\linewidth]{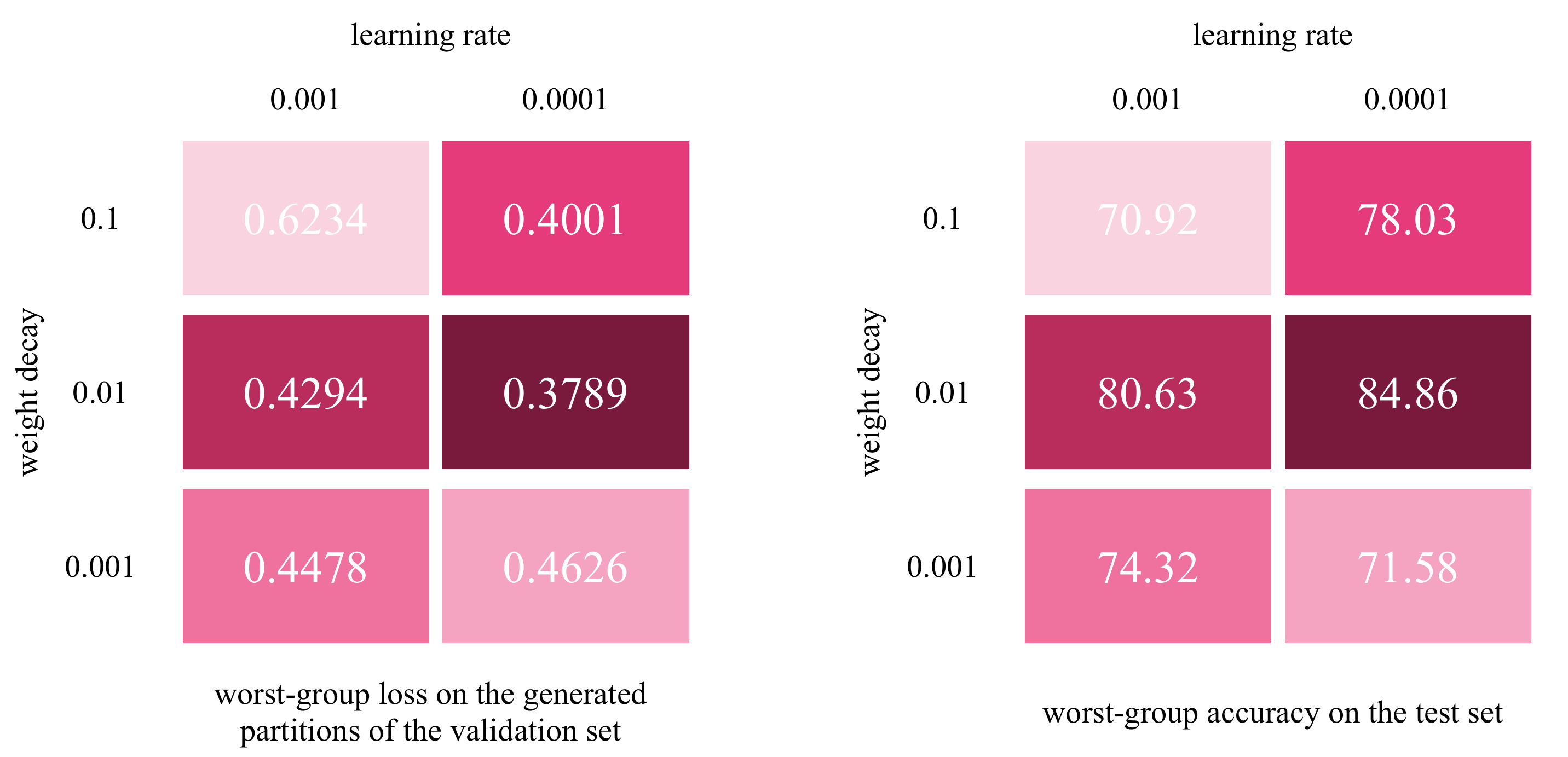}
  \end{center}
  \caption{Hyper-parameter selection for \textsc{tofu} on CelebA (averaged across 5 runs). We use our learned unstable feature representation $f_\mathcal{Z}$ to partition the validation set and use the worst-group validation loss as our hyper-parameter selection criteria. Empirically, we observe that this criteria correlates well with the model robustness on the testing data.}
  \label{fig:val_loss_test_acc}
\end{figure}

\textbf{\emph{How do we select the hyper-parameter for \textsc{tofu}?}} We cluster the validation data based on the learned unstable feature representation $f_\mathcal{Z}$ and use the worst-group loss as our early stopping and hyper-parameter selection criteria. Figure~\ref{fig:val_loss_test_acc} shows our hyper-parameter search space. We observe that our validation criteria correlates well with the robustness of the model on the testing data.

\begin{table}[t]
  \centering
  \caption{Ablation study on \textsc{mnist} and \textsc{beer review}.
  }\label{tab:ablation}
  \small
  \begin{tabular}{ccccccc}
  \toprule
  & \textsc{source} &\textsc{target} & \hspace{-1mm}$\textsc{finetune}_{\textsc{pi}}$\hspace{-1mm} & \hspace{-1mm}\textsc{ablation}\hspace{-1mm} & \textsc{tofu} \\
  \midrule
\multirow{2}{*}{\STAB{\rotatebox[origin=c]{90}{\textsc{mnist}}}}\hspace{-1mm}
& \textsc{odd}     & \textsc{even}   & 11.2 & 19.2 & \textbf{69.1} \\ \cmidrule{2-6}
& \textsc{even}    & \textsc{odd}    & 11.5 & 18.7 & \textbf{66.8} \\ \midrule
\multirow{6}{*}{\STAB{\rotatebox[origin=c]{90}{\hspace{-9mm}\textsc{beer review}}}}\hspace{-1mm}
& \textsc{look}    & \textsc{aroma}  & 53.7 & 71.4 & \textbf{75.9} \\ \cmidrule{2-6}
& \textsc{look}    & \textsc{palate} & 49.3 & 64.6 & \textbf{73.8} \\ \cmidrule{2-6} 
& \textsc{aroma}   & \textsc{look}   & 65.2 & 69.6 & \textbf{80.9} \\ \cmidrule{2-6} 
& \textsc{aroma}   & \textsc{palate} & 47.9 & 50.3 & \textbf{73.5} \\ \cmidrule{2-6} 
& \textsc{palate}  & \textsc{look}   & 64.3 & 70.1 & \textbf{81.0}  \\ \cmidrule{2-6} 
& \textsc{palate}  & \textsc{aroma}  & 54.5 & 57.9 & \textbf{76.9}
\\\bottomrule
  \end{tabular}
\end{table}

\textbf{\emph{Ablation study}}
The procedures in \textsc{tofu} are interconnected. For example, we cannot apply group DRO (\textbf{T.2}) without clustering the target examples (\textbf{T.1}). Nevertheless, we can empirically verify that \textsc{tofu} provides cleaner separation in the unstable feature space than a naive variant. 

In Theorem 1, we show that it is necessary to compare example pairs with the same label value, as opposed to contrasting all example pairs. Otherwise, the unstable features will be coupled with the \emph{task-specific} information. 
%Our theory shows that among examples pairs \emph{with the same label value}, those that share the same prediction outcome should also have similar unstable features. It is important to look at examples with the same label value; otherwise, the unstable features will be coupled with the task-specific information in their prediction results. 
To empirically support our design choice, we consider a variant of \textsc{tofu} that minimizes the hinge loss over all example pairs in \textbf{S.3}:

\begin{description}
    \item[\textbf{S.3} (\textsc{ablation})]
    Learn a feature representation $f_{\mathcal{Z}}$ by minimizing Eq (1) across all pairs of environments $E_i, E_j$:
    \[
    f_\mathcal{Z} = \arg\min
    \sum_{E_i\neq E_j} 
    \mathbb{E}_{X_1^\checkmark, X_2^\checkmark, X_3^\times}
    \left[\mathcal{L}_\mathcal{Z}(X_1^\checkmark, X_2^\checkmark, X_3^\times)\right],
    \]
    where batches $X_1^\checkmark, X_2^\checkmark$ are sampled uniformly from $E_j^{i\checkmark}$ and batch $X_3^\times$ is sampled uniformly from $E_j^{i\times}$.
\end{description}

We can think of \textsc{ablation} as directly transferring the partitions from the source task to the target task. Table~\ref{tab:ablation} presents the results on \textsc{mnist} and \textsc{beer review}. We observe that while \textsc{ablation} slightly improves over the fine-tuning baseline, it significantly underperforms \textsc{tofu} across all transfer settings. Figure~\ref{fig:ablation} further confirms that the feature space learned by \textsc{ablation} are not representative of the unstable color feature. We will include this ablation analysis in our updated version.

%\textbf{\emph{metric learning}}

\begin{figure}[t]
    \centering
    \includegraphics[width=0.6\linewidth]{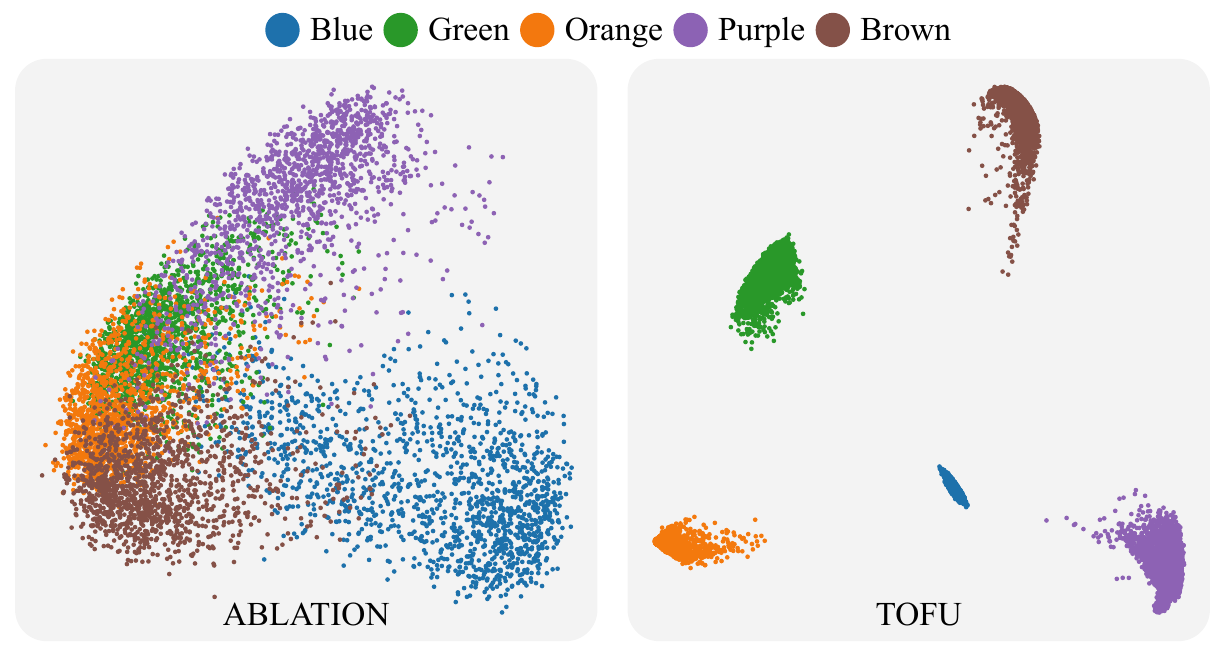}
    \caption{PCA visualizations of the feature representation $f_\mathcal{Z}$ for examples in \textsc{mnist} \textsc{even} (best viewed in color). $f_\mathcal{Z}$ is trained on \textsc{mnist} \textsc{odd}. Compared with \textsc{ablation} (left), \textsc{tofu} perfectly separates target examples based on their spurious features (color).}
    \label{fig:ablation}
\end{figure}
\section{Multiple source tasks}
\label{app:multi_source}
One major limitation of our work is that the source task and the target task need to share the same unstable features. While a single source task may not describe all unwanted unstable features, we can leverage multiple source tasks and combine their individual unstable features together.

\begin{table}[t]
  \centering
  \caption{Illustration of the tasks on \textsc{mnist} for multiple source tasks experiments. In the source tasks ($S_1, S_2, S_3$), we want to classify two digits where the label is spuriously correlated with a color pair (red-blue, red-green, blue-green). In the target task $T$, the goal is to learn a color-invariant model by using only one \emph{biased} environment $E_1^{\text{train}}$.
  }\label{tab:multi_source_data}
  \begin{tabular}{ccccc}
  \toprule
  Tasks & Labels & $E_1^{\text{train}}$  & $E_2^{\text{train}}$ & $E^{\text{test}}$  \\
  \midrule
  $S_1$ & \texttt{0} vs. \texttt{1} &
  \begin{tabular}{@{}c@{}}
    {\color{red}\texttt{000000000}}{\color{blue}\texttt{0}}\\
    {\color{blue}\texttt{111111111}}{\color{red}\texttt{1}}
  \end{tabular} &
  \begin{tabular}{@{}c@{}}
    {\color{red}\texttt{00000000}}{\color{blue}\texttt{00}}\\
    {\color{blue}\texttt{11111111}}{\color{red}\texttt{11}}
  \end{tabular} &
  \begin{tabular}{@{}c@{}}
    {\color{red}\texttt{0}}{\color{blue}\texttt{000000000}}\\
    {\color{blue}\texttt{1}}{\color{red}\texttt{111111111}}
  \end{tabular}
  \\ \midrule
  $S_2$ & \texttt{2} vs. \texttt{3} &
  \begin{tabular}{@{}c@{}}
    {\color{red}\texttt{222222222}}{\color{green}\texttt{2}}\\
    {\color{green}\texttt{333333333}}{\color{red}\texttt{3}}
  \end{tabular} &
  \begin{tabular}{@{}c@{}}
    {\color{red}\texttt{22222222}}{\color{green}\texttt{22}}\\
    {\color{green}\texttt{33333333}}{\color{red}\texttt{33}}
  \end{tabular} &
  \begin{tabular}{@{}c@{}}
    {\color{red}\texttt{2}}{\color{green}\texttt{222222222}}\\
    {\color{green}\texttt{3}}{\color{red}\texttt{333333333}}
  \end{tabular}
  \\ \midrule
    $S_3$ & \texttt{4} vs. \texttt{5} &
  \begin{tabular}{@{}c@{}}
    {\color{blue}\texttt{444444444}}{\color{green}\texttt{4}}\\
    {\color{green}\texttt{555555555}}{\color{blue}\texttt{5}}
  \end{tabular} &
  \begin{tabular}{@{}c@{}}
    {\color{blue}\texttt{444444444}}{\color{green}\texttt{44}}\\
    {\color{green}\texttt{55555555}}{\color{blue}\texttt{55}}
  \end{tabular} &
  \begin{tabular}{@{}c@{}}
    {\color{blue}\texttt{4}}{\color{green}\texttt{4444444444}}\\
    {\color{green}\texttt{5}}{\color{blue}\texttt{555555555}}
  \end{tabular} \\ \midrule
  $T$   & \texttt{6} vs. \texttt{7} vs. \texttt{8} &
  \begin{tabular}{@{}c@{}}
    {\color{red}\texttt{66666666}}{\color{blue}\texttt{6}}{\color{green}\texttt{6}}\\
    {\color{blue}\texttt{77777777}}{\color{green}\texttt{7}}{\color{red}\texttt{7}}\\
    {\color{green}\texttt{88888888}}{\color{red}\texttt{8}}{\color{blue}\texttt{8}}
  \end{tabular} & NA &
  \begin{tabular}{@{}c@{}}
    {\color{red}\texttt{66}}{\color{blue}\texttt{6666}}{\color{green}\texttt{6666}}\\
    {\color{blue}\texttt{77}}{\color{green}\texttt{7777}}{\color{red}\texttt{7777}}\\
    {\color{green}\texttt{88}}{\color{red}\texttt{8888}}{\color{blue}\texttt{8888}}
  \end{tabular}
  \\ \bottomrule
  \end{tabular}
\end{table}

\paragraph{Extending \textsc{tofu} to multiple source tasks}
We can naturally extend our algorithm by inferring a \emph{joint} unstable feature space across all source tasks.
\begin{description}
\item[\textbf{MS.1}] For each source task $S$ and for each source environment ${^S}E_i$, train an environment-specific classifier ${^S}f_i$.
\item[\textbf{MS.2}] For each source task $S$ and for each pair of environments ${^S}E_i$ and ${^S}E_j$, use classifier ${^S}f_i$ to partition ${^S}E_j$ into two sets: ${^S}E_{j}^{i \checkmark}$ and ${^S}E_j^{i \times}$, where ${^S}E_{j}^{i \checkmark}$ contains examples that ${^S}f_i$ predicted correctly and ${^S}E_j^{i \times}$ contains those predicted incorrectly.
\item[\textbf{MS.3}] Learn an unstable feature representation $f_{\mathcal{Z}}$ by minimizing Eq~\eqref{eq:hinge} across all source tasks $S$, all pairs of environments ${^S}E_i, {^S}E_j$ and all possible label value $y$: 
  \[f_\mathcal{Z} = \arg\min
  \sum_S \sum_{y, {^S}E_i\neq {^S}E_j} 
  \mathbb{E}_{X_1^\checkmark, X_2^\checkmark, X_3^\times}
  \left[\mathcal{L}_\mathcal{Z}(X_1^\checkmark, X_2^\checkmark, X_3^\times)\right],
  \]
  where batches $X_1^\checkmark, X_2^\checkmark$ are sampled uniformly from ${^S}E_j^{i\checkmark}|_y$ and batch $X_3^\times$ is sampled uniformly from ${^S}E_j^{i\times}|_y$ ($\cdot|_y$ denotes the subset of $\cdot$ with label value $y$).
\end{description}

On the target task, we use this joint unstable feature representation $f_{\mathcal{Z}}$ to generate clusters as in Section~\ref{sec:target}. Since $f_{\mathcal{Z}}$ is trained across the source tasks, the generated clusters are informative of all unstable features that are present in these tasks. By minimizing the worst-case risks across the clusters, we obtain the final stable classifier.

\paragraph{Experiment setup}
We design controlled experiments on MNIST to study the effect of having multiple source tasks.
We consider three source tasks: $S_1$ (\texttt{0} vs. \texttt{1}), $S_2$ (\texttt{2} vs \texttt{3}) and $S_3$ (\texttt{4} vs. \texttt{5}). For the target task $T$, the goal is to identify \texttt{6}, \texttt{7} and \texttt{8}.

Similar to Section~\ref{sec:setup}, we first generated the observed noisy label based on the digits. We then inject spurious color features to the input images. For $S_1$, $S_2$ and $S_3$, the noisy labels are correlated with \texttt{red}/\texttt{blue}, \texttt{red}/\texttt{green} and \texttt{blue}/\texttt{green} respectively. For the target task $T$, the three noisy labels (\texttt{6}/\texttt{7}/\texttt{8}) are correlated with all three colors \texttt{red}/\texttt{blue}/\texttt{green}. Table~\ref{tab:multi_source_data} illustrate the different spurious correlations across the tasks.

\paragraph{Baselines}
Since \textsc{erm} and \textsc{oracle} only depend on the target task, they are the same as we described in  Section~\ref{sec:setup}. For \textsc{reuse} and \textsc{finetune}, we first use multitask learning to learn a shared feature representation across all tasks. Specifically, for each source task, we first partition its data into subsets with opposite spurious correlations by contrasting its data environments $E_1^{\text{train}}$ and $E_2^{\text{train}}$~\cite{yujia2021predict}. We then train a joint model, with a different classifier head for each source task, by minimizing the worst-case risk over all these subsets for each source task. The shared feature representation is directly transferred to the target task. The baseline \textsc{multitask} is similar to \textsc{reuse} and \textsc{finetune}. The difference is that we jointly train the target task classifier together with all source tasks' classifiers.

\begin{table}[t]
  \centering
  \caption{Target task testing accuracy of different methods on \textsc{mnist} with different combinations of the source tasks (see Table~\ref{tab:multi_source_data} for an illustration of the tasks). Majority baseline is $33\%$. All methods are tuned based on a held-out validation set that follows from the same distribution as the target training data. Bottom right: standard deviation across 5 runs. Upper right: avg. source task testing performance (if applicable).
  }\label{tab:multi_source}
  \begin{tabular}{ccccccccc}
  \toprule
  \textsc{source} & \textsc{erm} & $\textsc{reuse}_\textsc{pi}$ & $\textsc{finetune}_\textsc{pi}$ & \textsc{multitask} & \textsc{tofu} & \textsc{oracle} \\
  \midrule
  $S_1$             & 26.8{\tiny${\pm2.4}$} & 34.7$_{\tiny\pm5.0}^{\tiny(72.0)}$  & 35.1$_{\tiny\pm2.4}^{\tiny(71.9)}$  &17.7$_{\tiny\pm0.3}^{\tiny(69.4)}$ &\textbf{57.3}{\tiny${\pm6.9}$} &72.7{\tiny${\pm0.7}$}\\ \midrule
  $S_2$             & 26.8{\tiny${\pm2.4}$} & 34.6$_{\tiny\pm1.7}^{\tiny(68.0)}$  & 31.0$_{\tiny\pm0.8}^{\tiny(66.7)}$  &14.6$_{\tiny\pm2.3}^{\tiny(74.5)}$ &\textbf{57.8}{\tiny${\pm8.3}$} &72.7{\tiny${\pm0.7}$}\\ \midrule
  $S_3$             & 26.8{\tiny${\pm2.4}$} & 34.1$_{\tiny\pm0.8}^{\tiny(70.2)}$  & 33.6$_{\tiny\pm0.7}^{\tiny(66.3)}$  &12.9$_{\tiny\pm3.4}^{\tiny(71.2)}$ &\textbf{49.8}{\tiny${\pm5.2}$} &72.7{\tiny${\pm0.7}$}\\ \midrule
  $S_1 + S_2$       & 26.8{\tiny${\pm2.4}$} & 34.0$_{\tiny\pm13.9}^{\tiny(67.9)}$ & 18.3$_{\tiny\pm3.2}^{\tiny(68.2)}$  &22.2$_{\tiny\pm3.0}^{\tiny(71.3)}$ &\textbf{52.9}{\tiny${\pm1.0}$} &72.7{\tiny${\pm0.7}$}\\ \midrule
  $S_1 + S_3$       & 26.8{\tiny${\pm2.4}$} & 49.9$_{\tiny\pm15.7}^{\tiny(70.3)}$ & 48.3$_{\tiny\pm15.3}^{\tiny(68.7)}$ &20.3$_{\tiny\pm2.8}^{\tiny(72.3)}$ &\textbf{53.4}{\tiny${\pm2.3}$} &72.7{\tiny${\pm0.7}$}\\ \midrule
  $S_2 + S_3$       & 26.8{\tiny${\pm2.4}$} & 49.5$_{\tiny\pm7.5}^{\tiny(71.3)}$  & 50.9$_{\tiny\pm12.0}^{\tiny(72.0)}$ &18.5$_{\tiny\pm7.5}^{\tiny(74.6)}$ &\textbf{53.4}{\tiny${\pm4.1}$} &72.7{\tiny${\pm0.7}$}\\ \midrule
  $S_1 + S_2 + S_3$ & 26.8{\tiny${\pm2.4}$} & 34.1$_{\tiny\pm16.4}^{\tiny(69.0)}$ & 40.3$_{\tiny\pm26.3}^{\tiny(68.5)}$ &26.4$_{\tiny\pm1.2}^{\tiny(71.0)}$ &\textbf{72.3}{\tiny${\pm1.5}$} &72.7{\tiny${\pm0.7}$}\\ \bottomrule
  \end{tabular}
\end{table}
\paragraph{Results}
Table~\ref{tab:multi_source} presents our results on learning from multiple source tasks. Compared with the baselines, \textsc{tofu} achieves the best performance across all 7 transfer settings.

We observe that having two tasks doesn't necessarily improve the target performance for $\textsc{tofu}$. This result is actually not surprising. For example, let's consider having two source tasks $S_1$ and $S_2$.
\textsc{tofu} learns to recognize \texttt{red} vs. \texttt{blue} from $S_1$ and \texttt{red} vs. \texttt{green} from $S_2$, but \textsc{tofu} doesn't know that \texttt{blue} should be separated from \texttt{green} in the unstable feature space. Therefore, we shouldn't expect to see any performance improvement when we combine $S_1$ and $S_2$. However, if we have one more source task $S_3$ which specifies the invariance between \texttt{blue} and \texttt{green}, \textsc{tofu} is able to achieve the oracle performance.

For the direct transfer baselines, we see that \textsc{multitask} simply learns to overfit the spurious correlation and performs similar to \textsc{erm}. \textsc{reuse} and \textsc{finetune} generally perform better when more source tasks are available. However, their testing performance vary a lot across different runs.
\section{Full results on CelebA}
\label{appendix:full}

\begin{figure}[t]
    \centering
    \includegraphics[width=0.6\textwidth]{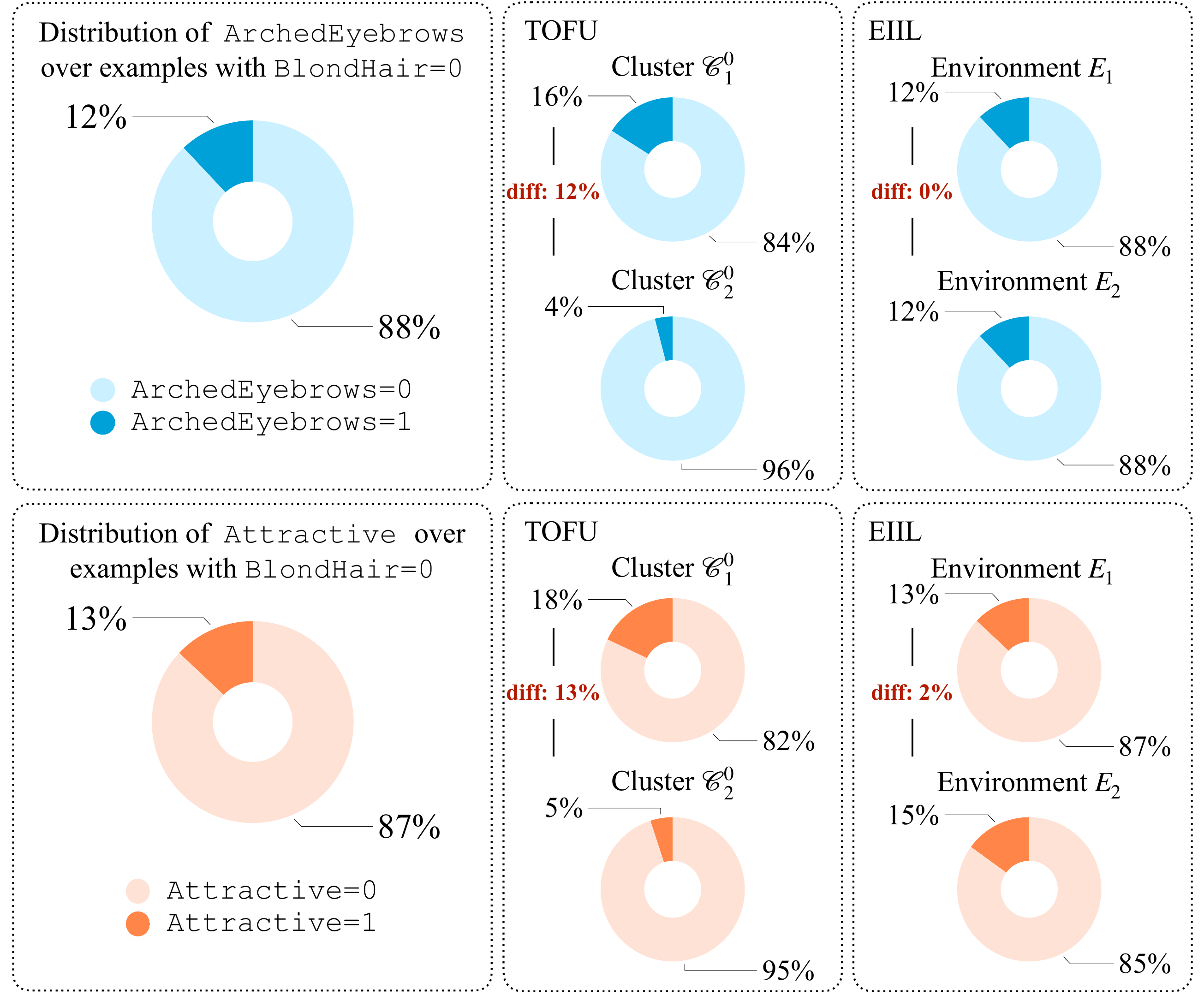}
    \caption{Visualization of the   unknown attribute \texttt{ArchedEyebrows} and \texttt{Attractive} on CelebA. Left:
        distributions of \texttt{ArchedEyebrows} and \texttt{Attractive} in the training data. Mid: partitions learned by \textsc{tofu}. Right: partitions learned by \textsc{eiil}.}
    \label{fig:celeba_visual_additional}
\end{figure}

\begin{landscape}

\begin{table}[t]\scriptsize
\centering
\caption{Worst-group accuracy on CelebA. 
  The source task is to predict \texttt{Eyeglasses} and the target task is to predict \texttt{BlondHair}. We use the attribute \texttt{Young} to define two environments: $E_1=\{\texttt{Young}=0\}$, $E_2=\{\texttt{Young}=1\}$. Both environments are available in the source task. In the target task, we only have access to $E_1$ during training and validation.
  }
\begin{tabular}{lcccccccccccccccc}
\toprule
                      Methods & \textsc{erm}   & \makecell{\textsc{finetune}\\ \textsc{pi}} & 
                      \makecell{\textsc{finetune}\\ \textsc{dann}} & \makecell{\textsc{finetune}\\ \textsc{c-dann}}  & \makecell{\textsc{finetune}\\ \textsc{mmd}}  & \makecell{\textsc{reuse}\\ \textsc{pi}}  & \makecell{\textsc{reuse}\\ \textsc{dann}}& \makecell{\textsc{reuse}\\ \textsc{c-dann}} & \makecell{\textsc{reuse}\\ \textsc{mmd}} & \makecell{\textsc{multi}\\ \textsc{task}} & \textsc{eiil}  & \textsc{george} & \textsc{lff}   & \textsc{m-ada} & \makecell{\textsc{dg-}\\ \textsc{mmld}} & \textsc{tofu}  \\
\midrule
Arched\_Eyebrows      & 75.43 & 71.86       & 65.38         & 73.85           & 76.07        & 53.71 & 59.56      & 56.02        & 48.91     & 69.66        & 64.71 & 74.73  & 45.41 & 64.61 & 69.51   & 85.66 \\
Attractive            & 75.00 & 72.73       & 63.35         & 75.61           & 74.33        & 52.13 & 62.03      & 57.78        & 48.46     & 72.73        & 64.43 & 73.66  & 47.67 & 67.33 & 68.42   & 88.30 \\
Bags\_Under\_Eyes     & 70.91 & 62.50       & 56.86         & 75.86           & 78.57        & 52.50 & 64.58      & 57.50        & 58.74     & 70.00        & 66.67 & 77.78  & 42.59 & 70.34 & 63.41   & 90.38 \\
Bald                  & 80.79 & 76.84       & 71.14         & 80.92           & 79.58        & 55.56 & 67.99      & 61.70        & 60.00     & 77.18        & 71.71 & 79.07  & 52.05 & 71.57 & 77.30   & 91.53 \\
Bangs                 & 76.06 & 69.33       & 63.59         & 77.11           & 71.76        & 51.15 & 64.67      & 53.42        & 59.29     & 70.22        & 65.29 & 76.74  & 48.04 & 63.54 & 66.88   & 88.70 \\
Big\_Lips             & 75.29 & 72.73       & 64.46         & 73.03           & 73.89        & 54.41 & 66.09      & 59.24        & 58.75     & 72.00        & 69.32 & 78.24  & 47.88 & 70.79 & 69.13   & 88.66 \\
Big\_Nose             & 80.65 & 71.43       & 68.29         & 79.17           & 76.47        & 54.33 & 63.27      & 58.90        & 51.16     & 76.59        & 71.43 & 78.76  & 48.84 & 70.93 & 68.29   & 88.89 \\
Black\_Hair           & 80.79 & 76.84       & 71.14         & 80.92           & 79.58        & 55.56 & 67.99      & 61.70        & 60.00     & 77.18        & 71.71 & 79.07  & 52.05 & 71.57 & 77.30   & 91.53 \\
Blurry                & 68.75 & 62.07       & 58.06         & 64.71           & 65.52        & 52.94 & 67.39      & 56.25        & 28.12     & 56.76        & 50.00 & 75.86  & 29.03 & 34.48 & 48.15   & 80.65 \\
Brown\_Hair           & 60.00 & 25.00       & 16.67         & 33.33           & 50.00        & 50.00 & 33.33      & 50.00        & 57.71     & 66.67        & 0.00  & 57.14  & 51.74 & 60.00 & 66.67   & 80.00 \\
Bushy\_Eyebrows       & 80.79 & 76.67       & 66.67         & 80.75           & 50.00        & 55.56 & 67.99      & 61.57        & 52.66     & 77.18        & 50.00 & 50.00  & 51.89 & 50.00 & 50.00   & 91.47 \\
Chubby                & 57.14 & 20.00       & 12.50         & 25.00           & 28.57        & 33.33 & 40.00      & 60.00        & 60.00     & 33.33        & 40.00 & 33.33  & 51.05 & 33.33 & 77.22   & 57.14 \\
Double\_Chin          & 64.29 & 50.00       & 70.83         & 80.00           & 64.29        & 50.00 & 60.00      & 60.00        & 41.67     & 50.00        & 55.56 & 50.00  & 51.62 & 54.55 & 30.00   & 87.50 \\
Eyeglasses            & 60.00 & 68.75       & 53.85         & 75.00           & 58.33        & 15.38 & 6.25       & 15.38        & 0.00      & 60.00        & 42.86 & 76.47  & 22.22 & 69.23 & 40.00   & 73.33 \\
Goatee                & 0.00  & 76.84       & 0.00          & 0.00            & 79.58        & 55.56 & 67.99      & 61.70        & 60.00     & 77.10        & 71.71 & 0.00   & 52.05 & 0.00  & 0.00    & 91.53 \\
Gray\_Hair            & 64.29 & 54.55       & 54.55         & 63.64           & 66.67        & 55.44 & 14.29      & 33.33        & 59.85     & 44.44        & 37.50 & 73.33  & 20.00 & 71.28 & 37.50   & 85.71 \\
Heavy\_Makeup         & 70.95 & 65.89       & 56.85         & 70.59           & 71.53        & 52.48 & 56.49      & 50.76        & 44.53     & 66.91        & 54.86 & 70.83  & 38.10 & 62.24 & 60.63   & 84.25 \\
High\_Cheekbones      & 65.96 & 62.96       & 58.62         & 72.34           & 75.61        & 47.31 & 55.29      & 50.53        & 45.88     & 66.33        & 54.00 & 68.82  & 37.21 & 51.09 & 62.96   & 86.73 \\
Male                  & 22.86 & 20.00       & 21.62         & 26.32           & 34.48        & 32.00 & 26.47      & 43.33        & 39.29     & 33.33        & 23.53 & 40.62  & 20.00 & 21.88 & 10.53   & 66.67 \\
Mouth\_Slightly\_Open & 75.47 & 73.64       & 68.10         & 77.86           & 79.55        & 45.61 & 64.71      & 57.02        & 57.28     & 76.03        & 67.44 & 78.33  & 48.57 & 66.39 & 76.47   & 91.01 \\
Mustache              & 80.79 & 76.84       & 71.14         & 80.92           & 79.58        & 55.56 & 67.99      & 61.70        & 60.00     & 77.18        & 71.71 & 79.07  & 52.05 & 71.57 & 77.30   & 91.53 \\
Narrow\_Eyes          & 74.58 & 76.13       & 59.70         & 78.87           & 78.39        & 52.70 & 67.24      & 60.26        & 52.83     & 71.67        & 68.66 & 67.74  & 50.79 & 58.21 & 69.49   & 87.04 \\
No\_Beard             & 0.00  & 76.75       & 0.00          & 0.00            & 79.58        & 0.00  & 67.99      & 0.00         & 60.00     & 77.10        & 71.62 & 0.00   & 51.89 & 0.00  & 33.33   & 91.50 \\
Oval\_Face            & 79.92 & 75.00       & 70.28         & 80.85           & 70.00        & 55.00 & 67.31      & 60.17        & 53.61     & 76.77        & 71.09 & 78.21  & 51.61 & 70.20 & 76.79   & 91.37 \\
Pale\_Skin            & 71.43 & 76.54       & 60.00         & 80.39           & 72.22        & 46.15 & 67.59      & 54.55        & 54.55     & 71.43        & 71.03 & 69.23  & 42.86 & 60.00 & 77.12   & 83.05 \\
Pointy\_Nose          & 77.72 & 73.30       & 65.82         & 77.03           & 75.13        & 52.69 & 67.98      & 60.87        & 54.30     & 75.26        & 66.00 & 77.42  & 48.28 & 69.19 & 71.58   & 90.20 \\
Receding\_Hairline    & 41.18 & 68.75       & 43.75         & 52.38           & 63.64        & 47.37 & 40.00      & 41.67        & 58.89     & 58.82        & 26.67 & 65.00  & 35.00 & 61.11 & 36.84   & 70.00 \\
Rosy\_Cheeks          & 78.49 & 75.11       & 67.47         & 79.41           & 79.20        & 54.22 & 65.18      & 56.06        & 39.39     & 75.89        & 68.50 & 78.17  & 38.36 & 68.80 & 74.68   & 87.69 \\
Sideburns             & 0.00  & 76.84       & 0.00          & 0.00            & 79.58        & 55.56 & 67.99      & 61.70        & 60.00     & 77.10        & 71.71 & 0.00   & 52.05 & 0.00  & 0.00    & 91.53 \\
Smiling               & 72.16 & 71.74       & 62.77         & 73.08           & 77.78        & 45.10 & 57.95      & 52.08        & 51.58     & 73.00        & 61.95 & 75.00  & 36.46 & 62.63 & 69.41   & 90.09 \\
Straight\_Hair        & 79.37 & 65.31       & 53.23         & 68.75           & 68.63        & 54.98 & 66.15      & 57.69        & 57.96     & 76.54        & 68.25 & 71.93  & 50.00 & 57.89 & 71.74   & 84.48 \\
Wavy\_Hair            & 77.19 & 69.01       & 67.47         & 76.24           & 74.68        & 55.47 & 67.82      & 60.14        & 58.99     & 75.33        & 70.19 & 75.80  & 50.60 & 65.06 & 76.43   & 87.95 \\
Wearing\_Earrings     & 71.88 & 70.63       & 64.00         & 72.73           & 75.16        & 54.65 & 60.00      & 55.06        & 52.00     & 70.29        & 64.50 & 76.97  & 43.26 & 66.48 & 70.97   & 87.36 \\
Wearing\_Hat          & 80.79 & 76.84       & 71.14         & 80.92           & 79.58        & 55.56 & 67.99      & 61.70        & 60.00     & 77.18        & 71.71 & 79.07  & 52.05 & 71.57 & 77.30   & 91.53 \\
Wearing\_Lipstick     & 51.32 & 46.97       & 40.00         & 46.75           & 56.06        & 46.38 & 45.12      & 46.38        & 41.79     & 50.00        & 41.89 & 66.67  & 32.89 & 50.00 & 42.11   & 76.32 \\
Wearing\_Necklace     & 76.12 & 70.22       & 63.45         & 72.77           & 74.72        & 51.87 & 62.87      & 54.60        & 56.10     & 68.95        & 64.53 & 75.00  & 49.73 & 65.05 & 70.83   & 84.65 \\
Wearing\_Necktie      & 0.00  & 20.00       & 0.00          & 0.00            & 20.00        & 50.00 & 16.67      & 50.00        & 0.00      & 0.00         & 0.00  & 0.00   & 20.00 & 0.00  & 0.00    & 57.14 \\
5\_Clock\_Shadow      & 0.00  & 0.00        & 0.00          & 50.00           & 0.00         & 0.00  & 0.00       & 61.70        & 59.04     & 0.00         & 66.67 & 79.00  & 0.00  & 50.00 & 0.00    & 91.53 \\\midrule
Avg                   & 61.01 & 63.07       & 50.60         & 62.03           & 66.80        & 47.58 & 55.27      & 53.22        & 50.61     & 64.37        & 57.62 & 63.34  & 42.52 & 54.55 & 55.69   & 84.86
\\\bottomrule
\end{tabular}
\end{table}

\begin{table}[t]\scriptsize
\centering
\caption{Average-group accuracy on CelebA. 
  The source task is to predict \texttt{Eyeglasses} and the target task is to predict \texttt{BlondHair}. We use the attribute \texttt{Young} to define two environments: $E_1=\{\texttt{Young}=0\}$, $E_2=\{\texttt{Young}=1\}$. Both environments are available in the source task. In the target task, we only have access to $E_1$ during training and validation.
  }
\begin{tabular}{lcccccccccccccccc}
\toprule
                      Methods & \textsc{erm}   & \makecell{\textsc{finetune}\\ \textsc{pi}} & 
                      \makecell{\textsc{finetune}\\ \textsc{dann}} & \makecell{\textsc{finetune}\\ \textsc{c-dann}}  & \makecell{\textsc{finetune}\\ \textsc{mmd}}  & \makecell{\textsc{reuse}\\ \textsc{pi}}  & \makecell{\textsc{reuse}\\ \textsc{dann}}& \makecell{\textsc{reuse}\\ \textsc{c-dann}} & \makecell{\textsc{reuse}\\ \textsc{mmd}} & \makecell{\textsc{multi}\\ \textsc{task}} & \textsc{eiil}  & \textsc{george} & \textsc{lff}   & \textsc{m-ada} & \makecell{\textsc{dg-}\\ \textsc{mmld}} & \textsc{tofu}  \\
\midrule
Arched\_Eyebrows      & 88.52 & 87.02    & 83.89         & 88.90           & 88.80        & 64.05 & 72.44      & 67.07        & 59.80     & 86.91        & 85.12 & 87.89  & 60.23 & 83.33 & 87.38   & 91.47 \\
Attractive            & 88.94 & 87.34    & 84.98         & 89.39           & 89.74        & 64.85 & 72.26      & 67.90        & 61.51     & 87.44        & 85.96 & 87.70  & 60.16 & 83.59 & 87.50   & 92.76 \\
Bags\_Under\_Eyes     & 87.09 & 84.10    & 81.34         & 88.14           & 88.61        & 66.88 & 73.83      & 68.33        & 63.11     & 85.21        & 83.90 & 87.97  & 60.72 & 85.34 & 84.78   & 92.41 \\
Bald                  & 92.50 & 90.86    & 89.31         & 92.60           & 92.41        & 68.41 & 80.26      & 74.62        & 62.83     & 91.04        & 89.86 & 91.91  & 67.94 & 88.81 & 91.60   & 94.73 \\
Bangs                 & 88.69 & 86.76    & 83.95         & 88.83           & 88.91        & 65.25 & 73.53      & 68.06        & 61.32     & 86.98        & 84.70 & 87.45  & 59.25 & 83.02 & 86.92   & 91.96 \\
Big\_Lips             & 88.99 & 87.23    & 84.19         & 89.04           & 88.87        & 64.42 & 73.86      & 68.45        & 61.30     & 87.24        & 84.91 & 87.87  & 60.98 & 83.13 & 87.56   & 91.78 \\
Big\_Nose             & 89.17 & 85.87    & 83.58         & 88.71           & 88.20        & 66.33 & 73.79      & 72.28        & 60.50     & 87.47        & 84.89 & 88.52  & 61.76 & 84.47 & 85.73   & 92.24 \\
Black\_Hair           & 92.36 & 90.98    & 89.16         & 92.46           & 92.33        & 69.10 & 78.04      & 73.11        & 61.93     & 90.88        & 89.76 & 91.64  & 65.59 & 88.79 & 91.49   & 94.59 \\
Blurry                & 85.84 & 83.87    & 81.31         & 85.65           & 84.64        & 64.00 & 73.70      & 67.50        & 57.77     & 82.58        & 79.89 & 87.05  & 57.64 & 75.71 & 80.97   & 89.54 \\
Brown\_Hair           & 83.48 & 74.15    & 70.11         & 77.46           & 81.59        & 62.60 & 63.56      & 65.43        & 64.43     & 84.43        & 66.88 & 82.78  & 65.16 & 79.39 & 84.58   & 89.42 \\
Bushy\_Eyebrows       & 92.51 & 93.35    & 83.67         & 94.41           & 81.87        & 68.45 & 77.33      & 80.75        & 59.30     & 91.05        & 79.81 & 81.49  & 75.26 & 79.29 & 81.25   & 95.90 \\
Chubby                & 83.66 & 73.26    & 70.27         & 75.81           & 76.60        & 59.84 & 68.50      & 69.98        & 62.31     & 76.28        & 77.23 & 77.25  & 74.76 & 74.74 & 93.51   & 84.85 \\
Double\_Chin          & 85.40 & 80.78    & 86.50         & 89.03           & 85.46        & 63.08 & 73.40      & 69.35        & 58.24     & 80.44        & 81.00 & 81.38  & 65.47 & 80.04 & 76.46   & 92.14 \\
Eyeglasses            & 84.55 & 85.51    & 80.49         & 87.93           & 83.80        & 56.59 & 63.18      & 62.01        & 54.50     & 83.33        & 78.39 & 87.63  & 56.20 & 83.68 & 78.97   & 89.34 \\
Goatee                & 69.44 & 91.19    & 67.04         & 69.51           & 92.41        & 70.67 & 78.34      & 73.38        & 61.62     & 93.14        & 89.86 & 69.00  & 66.23 & 66.54 & 68.77   & 94.74 \\
Gray\_Hair            & 84.77 & 81.38    & 79.83         & 84.56           & 85.53        & 64.70 & 60.83      & 61.60        & 64.12     & 77.40        & 76.56 & 86.77  & 56.57 & 86.24 & 77.99   & 90.88 \\
Heavy\_Makeup         & 87.65 & 86.04    & 82.84         & 87.97           & 88.33        & 64.24 & 70.58      & 66.23        & 59.03     & 85.76        & 83.80 & 86.77  & 57.57 & 82.23 & 85.58   & 91.89 \\
High\_Cheekbones      & 86.81 & 84.89    & 82.24         & 87.76           & 88.11        & 63.41 & 72.29      & 67.73        & 60.11     & 85.31        & 82.59 & 86.57  & 59.92 & 80.71 & 85.33   & 91.98 \\
Male                  & 75.65 & 73.61    & 72.84         & 76.65           & 78.35        & 58.93 & 64.02      & 63.84        & 57.06     & 76.65        & 73.87 & 78.95  & 53.36 & 72.08 & 71.41   & 86.09 \\
Mouth\_Slightly\_Open & 88.26 & 86.66    & 83.82         & 88.77           & 88.73        & 63.64 & 73.97      & 68.89        & 61.89     & 86.63        & 84.59 & 87.93  & 61.74 & 83.20 & 87.41   & 92.68 \\
Mustache              & 92.50 & 91.18    & 89.31         & 92.60           & 92.42        & 70.96 & 80.02      & 74.32        & 62.43     & 90.88        & 89.70 & 91.74  & 67.73 & 88.35 & 91.61   & 94.90 \\
Narrow\_Eyes          & 87.39 & 87.46    & 82.06         & 88.62           & 89.73        & 65.57 & 74.24      & 69.02        & 62.02     & 85.66        & 84.32 & 85.98  & 61.69 & 81.11 & 85.95   & 91.55 \\
No\_Beard             & 69.29 & 93.18    & 66.99         & 69.43           & 92.17        & 52.43 & 77.55      & 55.04        & 62.96     & 92.90        & 92.21 & 68.85  & 75.04 & 66.44 & 77.04   & 95.72 \\
Oval\_Face            & 90.37 & 89.64    & 85.11         & 89.30           & 86.80        & 64.36 & 74.17      & 71.06        & 65.05     & 87.28        & 85.68 & 89.15  & 61.67 & 85.11 & 88.25   & 93.15 \\
Pale\_Skin            & 84.86 & 87.29    & 78.96         & 87.83           & 87.05        & 62.61 & 74.97      & 66.94        & 61.23     & 82.93        & 86.75 & 85.24  & 62.05 & 81.10 & 87.52   & 89.09 \\
Pointy\_Nose          & 88.96 & 87.15    & 84.79         & 89.31           & 89.70        & 64.31 & 73.78      & 68.35        & 62.87     & 86.78        & 85.66 & 87.19  & 62.81 & 84.03 & 88.31   & 92.46 \\
Receding\_Hairline    & 79.83 & 85.35    & 78.16         & 82.79           & 85.12        & 62.63 & 68.56      & 65.37        & 65.27     & 82.72        & 74.29 & 85.13  & 59.37 & 81.64 & 78.43   & 88.25 \\
Rosy\_Cheeks          & 89.36 & 87.33    & 87.21         & 90.18           & 88.66        & 64.09 & 73.83      & 67.34        & 57.29     & 86.74        & 87.18 & 88.25  & 62.42 & 84.55 & 88.97   & 92.44 \\
Sideburns             & 69.45 & 91.20    & 67.04         & 69.38           & 92.41        & 71.82 & 77.79      & 73.31        & 63.78     & 93.02        & 89.87 & 69.00  & 66.78 & 66.28 & 68.78   & 94.91 \\
Smiling               & 87.74 & 86.29    & 82.97         & 88.00           & 88.47        & 63.41 & 72.87      & 68.04        & 61.17     & 86.21        & 83.88 & 87.58  & 60.19 & 82.50 & 86.38   & 92.49 \\
Straight\_Hair        & 88.96 & 84.84    & 81.07         & 87.03           & 86.62        & 64.73 & 74.71      & 68.96        & 63.46     & 87.30        & 84.27 & 86.76  & 62.28 & 81.02 & 86.64   & 91.13 \\
Wavy\_Hair            & 88.67 & 86.64    & 83.64         & 88.89           & 88.79        & 65.01 & 73.11      & 67.84        & 62.01     & 86.55        & 84.70 & 87.69  & 60.87 & 83.03 & 87.14   & 92.32 \\
Wearing\_Earrings     & 88.30 & 86.72    & 83.74         & 88.77           & 88.57        & 64.06 & 73.10      & 68.00        & 61.03     & 86.66        & 84.54 & 87.49  & 59.64 & 83.36 & 86.97   & 92.06 \\
Wearing\_Hat          & 92.54 & 90.68    & 89.34         & 92.63           & 92.44        & 69.06 & 79.95      & 73.52        & 69.47     & 91.09        & 89.88 & 91.67  & 66.86 & 89.02 & 91.63   & 94.69 \\
Wearing\_Lipstick     & 82.97 & 81.17    & 78.10         & 82.33           & 84.24        & 62.53 & 68.30      & 64.67        & 57.33     & 81.22        & 79.07 & 85.07  & 56.03 & 78.93 & 79.89   & 88.85 \\
Wearing\_Necklace     & 88.63 & 87.57    & 84.75         & 89.72           & 88.61        & 63.96 & 72.85      & 67.88        & 61.19     & 87.65        & 85.54 & 87.75  & 58.05 & 83.66 & 87.18   & 91.03 \\
Wearing\_Necktie      & 69.35 & 73.32    & 67.03         & 69.56           & 74.40        & 63.62 & 63.11      & 67.21        & 46.32     & 68.03        & 67.50 & 68.88  & 56.13 & 66.31 & 68.81   & 84.64 \\
5\_Clock\_Shadow      & 69.38 & 68.21    & 66.92         & 81.87           & 69.33        & 52.69 & 57.14      & 72.06        & 70.87     & 68.06        & 83.83 & 93.66  & 51.29 & 78.39 & 68.77   & 93.93 \\
\midrule
Avg                   & 85.07 & 85.27    & 80.49         & 85.57           & 86.81        & 64.14 & 72.31      & 68.56        & 61.27     & 85.21        & 83.22 & 85.04  & 62.04 & 80.77 & 83.51   & 91.71
\\\bottomrule
\end{tabular}
\end{table}
\end{landscape}

% Acknowledgements should only appear in the accepted version.
% \section*{Acknowledgements}
% \textbf{Do not} include acknowledgements in the initial version of the paper submitted for blind review.

% In the unusual situation where you want a paper to appear in the
% references without citing it in the main text, use \nocite
% \nocite{langley00}

%%%%%%%%%%%%%%%%%%%%%%%%%%%%%%%%%%%%%%%%%%%%%%%%%%%%%%%%%%%%%%%%%%%%%%%%%%%%%%%
%%%%%%%%%%%%%%%%%%%%%%%%%%%%%%%%%%%%%%%%%%%%%%%%%%%%%%%%%%%%%%%%%%%%%%%%%%%%%%%
% APPENDIX
%%%%%%%%%%%%%%%%%%%%%%%%%%%%%%%%%%%%%%%%%%%%%%%%%%%%%%%%%%%%%%%%%%%%%%%%%%%%%%%
%%%%%%%%%%%%%%%%%%%%%%%%%%%%%%%%%%%%%%%%%%%%%%%%%%%%%%%%%%%%%%%%%%%%%%%%%%%%
\end{document}